\documentclass{article}
\usepackage[numbers,sort&compress]{natbib}
\usepackage{amsmath}
\usepackage{amssymb}
\usepackage{multirow}
\usepackage{amsfonts}
\usepackage{mathrsfs}
\usepackage{booktabs}
\usepackage{authblk}
\usepackage{algorithmicx}
\usepackage{algorithm}
\usepackage{booktabs}
\usepackage{algpseudocode}
\usepackage{comment}
\usepackage{amsthm}
\usepackage{latexsym}
\usepackage{graphicx}
\usepackage{rotating}
\usepackage{subfig}
\usepackage{color}
\usepackage{varioref}
\usepackage{longtable}
\usepackage{epstopdf}
\allowdisplaybreaks[4]
\usepackage[colorlinks,linkcolor=blue]{hyperref}

\newtheorem{theorem}{Theorem}
\newtheorem{assumption}{Assumption}
\newtheorem{proposition}{Proposition}
\newtheorem{definition}{Definition}
\newtheorem{remark}{Remark}

\usepackage{geometry}
\begin{document}

\title{An Accelerated Alternating Partial Bregman Algorithm for ReLU-based Matrix Decomposition}

\author{Qingsong Wang\thanks{School of Mathematics and Computational Science, Xiangtan University, Xiangtan, 411105, China. Email: nothing2wang@hotmail.com}
	, \hskip 0.2cm Yunfei Qu\thanks{LMIB, School of Mathematical Sciences, Beihang University, Beijing 100191, China. Email: yunfei$\_$math@hotmail.com}
	, \hskip 0.2cm Chunfeng Cui\thanks{LMIB, School of Mathematical Sciences, Beihang University, Beijing 100191, China. Email: chunfengcui@buaa.edu.cn}
	, \hskip 0.2cm Deren Han\thanks{Corresponding author. LMIB, School of Mathematical Sciences, Beihang University, Beijing 100191, China. Email: handr@buaa.edu.cn}
}

\maketitle

\begin{abstract} 
Despite the remarkable success of low-rank estimation in data mining, its effectiveness diminishes when applied to data that inherently lacks low-rank structure. To address this limitation, in this paper, we focus on non-negative sparse matrices and aim to investigate the intrinsic low-rank characteristics of the rectified linear unit (ReLU) activation function.   
We first propose a novel nonlinear matrix decomposition framework incorporating a comprehensive regularization term designed to simultaneously promote useful structures in clustering and compression tasks, such as low-rankness, sparsity, and non-negativity in the resulting factors. This formulation presents significant computational challenges due to its multi-block structure, non-convexity, non-smoothness, and the absence of global gradient Lipschitz continuity. To address these challenges, we develop an accelerated alternating partial Bregman proximal gradient method (AAPB), whose distinctive feature lies in its capability to enable simultaneous updates of multiple variables.  Under mild and theoretically justified assumptions, we establish both sublinear and global convergence properties of the proposed algorithm. Through careful selection of kernel generating distances tailored to various regularization terms, we derive corresponding closed-form solutions while maintaining the  $L$-smooth adaptable property always holds for any $L\ge 1$. Numerical experiments, on graph regularized clustering and sparse NMF basis compression confirm the effectiveness of our model and algorithm. 

\end{abstract}


\maketitle
\section{Introduction}

Exploring the low-rank structure inherent in real-world data is a fundamental technique within the applied mathematics and machine learning community \cite{Bishop06}. 
Specifically, given a matrix $M\in\mathbb R^{m\times n}$, the low-rank approximation method assumes that $M$ 
can be effectively approximated by a low-rank matrix, i.e., 
\begin{eqnarray}\label{eq:M_lowrank}
	M\approx X, \ \ X \text{ is low-rank}.
\end{eqnarray} 
Prominent low-rank decomposition methodologies encompass two fundamental approaches: the truncated singular value decomposition (TSVD) \cite{EckartY36} and non-negative matrix factorization (NMF) \cite{LeeS99, Gillis20}. These established techniques have been widely adopted in various applications due to their theoretical foundations and practical effectiveness \cite{WrightM22}. 

However, in practical applications, the observed data often deviates from the ideal low-rank assumption, thereby limiting the applicability of models based on the assumption in \eqref{eq:M_lowrank}. This limitation becomes particularly pronounced when dealing with large-scale, non-negative, and sparse matrices \cite{Saul22, Saul23}. A notable example is the identity matrix, which, despite being highly sparse, inherently possesses full rank. Direct application of low-rank approximation to such data structures may consequently fail to capture and exploit their intrinsic structural properties. 
Given the fundamental importance and ubiquitous applications of sparse non-negative matrices in diverse scientific domains, investigating their inherent low-rank properties through activation functions has garnered significant research attention. These matrices play pivotal roles in numerous applications, including but not limited to: dimensionality reduction \cite{LiuLLLSL17, CheWC23}, hyperspectral unmixing \cite{XiongZLQ20, GuoMYCLG22}, advanced signal processing \cite{HumbertBOKV21, MarminGF23}, and community detection in complex networks \cite{WuJKH18, LiHW18}. This widespread applicability has motivated extensive research into understanding and exploiting their underlying low-rank structures through activation function-based approaches.

A deeper understanding of the intricate relationship between low-rankness and sparsity can be achieved by extending beyond conventional linear algebraic decomposition methods \cite{UdellHZB16, WrightM22}. In a significant contribution to this field, Saul \cite{Saul22} investigated the intrinsic low-rank characteristics of non-negative sparse matrices through an innovative approach. Specifically, the study focused on developing a nonlinear matrix decomposition (NMD) framework that, given a non-negative sparse matrix $M\in\mathbb{R}^{m\times n}$, seeks to identify a real-valued matrix $X$  of equal or lower rank than $M$ as follows,   
\begin{eqnarray}\label{equ:NMD_approx}
	M\approx \max(0, X), \ \ X \text{ is low-rank}, 
\end{eqnarray} 
where $\max(0,\cdot)$ is the rectified linear unit (ReLU) function, 
which is widely-used as the activation function for deep neural networks \cite{GoodfellowBC16}. This indicates that when $M$ is sparse, we are able to explore a large space where the zero entries may be any non-positive values and seek a low-rank matrix \cite{Saul23}. 
A particularly insightful example from \cite{Saul22} demonstrates the remarkable capability of this approach: for an $n \times n$ identity matrix $M$, there exists a matrix $X$ satisfying $M = \max(0, X)$ with $\text{rank}(X) \le 3$ regardless of the dimension $n$. This striking result, maintaining rank-3 decomposition for identity matrices of arbitrary size, has propelled significant research interest in NMD, as evidenced by recent studies \cite{Saul22, MazumdarR19, Wang0DZHZL23}.

Mathematically, the ReLU-based NMD (ReLU-NMD) model approximates the expression in \eqref{equ:NMD_approx} by minimizing a least squares objective function, while predefinedly setting the rank of $X$ to be a positive integer $r$ that is less than  $\min(m,n)$,    i.e.,
\begin{eqnarray}
	\begin{aligned}
		\min_{X\in\mathbb{R}^{m\times n}}\, &\frac{1}{2}\|M-\max(0,X)\|^{2}_{F},\\
		\text{s.t.}\quad\, &\text{rank}(X)=r. 
	\end{aligned}\label{NMD}
\end{eqnarray} 
For notational simplicity, we denote the index set comprising zero and positive elements in $M$ as 
\begin{equation}\label{I_0}
	I_{0}:=\{(i,j)\, |\, M_{ij}=0\}  
\end{equation}
and
\begin{equation}\label{I_1}
	I_{+}:=\{ (i,j) \,|\, M_{ij}>0\},
\end{equation}
respectively.

The objective function delineated in \eqref{NMD} is characterized by the absence of both differentiability and convexity. Furthermore, the intrinsic nonlinearity introduced by the ReLU function leads to \eqref{NMD} has no closed-form solution and exacerbates the complexity of numerical methods.
In light of these challenges, Saul \cite{Saul22} subsequently proposed an alternative formulation of ReLU-NMD, incorporating a slack variable $W$, defined as 
\begin{eqnarray}
	\begin{aligned}
		\underset{X,W}{\min}\, &\  \frac{1}{2}\|W-X\|^{2}_{F},\\
		\text{s.t.}\,\,\, &\   \text{rank}(X)=r,\,\,\max(0,W)=M. \label{NMD-01}
	\end{aligned}
\end{eqnarray}
The main advantage of this new formulation is that   
the low-rank constraints and the nonlinear activation function are separate, allowing for the exploration of new solution strategies. 
However, almost all algorithms for solving \eqref{NMD-01}, such as  \cite{Saul22, Saul23},   require computing a rank-$r$ TSVD \cite{EckartY36} at each iteration, which can be computationally expensive, especially for large matrices.

To circumvent this computationally expensive step, Seraghiti et al. 
\cite{SeraghitiAVPG23} suggested replacing  {$X\in\mathbb{R}^{m\times n}$} with the product $UV$, where $U\in\mathbb{R}^{m\times r}$ and $V\in\mathbb{R}^{r\times n}$. Consequently, the problem \eqref{NMD-01} can be reformulated as 
\begin{eqnarray}
	\begin{aligned}
		\underset{U,V,W}{\min}\, &\ \frac{1}{2}\|W-UV\|^{2}_{F},\\
		\text{s.t.}\,\,\, &\  \max(0,W)=M. 
	\end{aligned}\label{NMD-MF}
\end{eqnarray}
A three-block nonlinear matrix decomposition algorithm (3B-NMD) was proposed in \cite{SeraghitiAVPG23} to solve this optimization problem \eqref{NMD-MF}. The advantage of this algorithm is that each subproblem has a closed-form solution, and extrapolation techniques can be implemented to accelerate the convergence. 
Their experiments showed remarkable results compared with the baselines, but both $U$- and $V$-subproblems might experience numerical instability. 

Recently, Wang et al. \cite{WangCH24a} explored a Tikhonov-regularized version of \eqref{NMD-MF} represented as 
\begin{eqnarray}
	\begin{aligned}
		\underset{U,V,W}{\min}\, &\  \frac{1}{2}\|W-UV\|^{2}_{F}+\frac{\eta}{2}\|U\|_{F}^{2}+\frac{\eta}{2}\|V\|_{F}^{2},\\
		\text{s.t.}\,\,\, &\ \max(0,W)=M, 
	\end{aligned}\label{NMD-T-MF}
\end{eqnarray}
where $\eta>0$ is a predefined parameter. One advantage of model \eqref{NMD-T-MF} is its ability to prevent overly aggressive steps in the alternating least-squares procedures in the 3B-NMD algorithm for the problem \eqref{NMD-MF}. 
However, all the above models fail to consider more comprehensive scenarios of the matrix factors $U$ and $V$, such as low-rankness, sparsity or non-negativity. These structures are crucial for mitigating over-fitting, bolstering generalization capabilities, and enhancing solution stability.

To bridge these gaps, this paper proposes a more generalized and inclusive model as follows, 

\begin{eqnarray}
	\begin{aligned}
		\underset{U,V,W}{\min}\, &\  \frac{1}{2}\|W-UV\|^{2}_{F}+H_{1}(U)+H_{2}(V),\\
		\text{s.t.}\,\,\, &\ \max(0,W)=M, 
	\end{aligned}\label{NMD-T-Obj}
\end{eqnarray}
where functions $H_{1}(\cdot), H_{2}(\cdot)$ represent regularization terms, such as nuclear norm \cite{RechtFP10},  $l_2$ norm \cite{GolubHO99}, $l_1$ norm \cite{Tibshirani96}, or non-convex sparse terms \cite{FanL01, Zhang10},  which play the roles of regularizer promoting low-complexity structures such as low-rankness, sparsity or non-negativity of the solution.  


The optimization problem \eqref{NMD-T-Obj}  presents significant computational challenges due to its multi-block structure, non-convexity, non-smoothness, and the absence of global gradient Lipschitz continuity. One state-of-the-art approach to solving the multi-block problem is the
alternating minimization method, which iteratively minimizes the objective function by fixing all variables except one. While leveraging the block structure can be intuitive and efficient, convergence is only assured when the minimum in each step is uniquely determined \cite{AttouchBRS10}. To ensure the uniqueness of the minimum at each step, Attouch et al. \cite{AttouchBRS10} introduced the proximal alternating minimization (PAM) method. However, the subproblems in PAM typically lack closed-form solutions, limiting the algorithm's practical performance. To address this limitation, Bolte et al. \cite{BolteST14} developed an approximation method for PAM via proximal linearization, resulting in the proximal alternating linearized minimization (PALM) algorithm. Several inertial versions have been proposed to enhance the numerical performance of PALM, including iPALM \cite{PockS16}, GiPALM \cite{GaoCH20}, and NiPALM \cite{WangH23a}. 
However, the convergence analysis of the alternating schemes usually requires the variable in the objective function to exhibit a global Lipschitz continuous gradient, which fails to hold for the matrix decomposition problems \cite{BolteSTV18First, MukkamalaO19}.

To overcome these limitations and build the convergence in the absence of global Lipschitz continuous gradient, a generalization of the classical gradient Lipschitz continuity concept becomes essential. This generalization was initially proposed in \cite{BirnbaumDX11} and subsequently developed for nonconvex optimization frameworks in \cite{BolteSTV18First}. The fundamental approach utilizes a generalized proximity measure based on Bregman distance, which serves as the foundation for the Bregman proximal gradient (BPG) method. This method extends the conventional proximal gradient approach by substituting Euclidean distances with Bregman distances as proximity measures. The corresponding convergence theory is established through a generalized Lipschitz condition, known as the $L$-smooth adaptable ($L$-smad) property \cite{BolteSTV18First}. Compared to PALM-type methods, the BPG-type algorithm offers several advantages \cite{MukkamalaO19}. First, BPG methods rely on a global 
$L$-smad constant, which is computed only once, simplifying the process. Second, block updates can be simultaneous, making BPG particularly effective for many matrix decomposition optimization problems \cite{mukkamala2022bregman}. Additionally, BPG-type methods are highly versatile, allowing for a broader range of applications. This flexibility opens up opportunities to design new loss functions and regularizers, without being limited to Lipschitz continuous gradients.

Despite the growth of BPG-type algorithms in recent literature, none have effectively exploited the specific structure of problem \eqref{NMD-T-Obj}. Existing methods, including BPG \cite{BolteSTV18First}, inertial BPG (iBPG)  \cite{MukkamalaOPS20}, DCAe \cite{PhaA24}, and BPDCA \cite{TakahashiFT22}, are primarily designed for single-block optimization problems, although several of these methods have been extended to address two-block variable applications in their respective studies.  In contrast, our problem \eqref{NMD-T-Obj} presents unique computational advantages when its multi-block structure is properly leveraged. Although some studies, such as BMME \cite{HienPGAP22}, have addressed multi-block optimization by applying BPG to individual subproblems, our approach is fundamentally distinguished by the $W$-subproblem in \eqref{NMD-T-Obj} admitting a closed-form solution - a distinctive feature that sets our method apart from existing frameworks.  A comprehensive summary is presented in Table  \ref{summary_table}.

\begin{table*}[h]
	\fontsize{8}{15}\selectfont
	\caption{Summary of mentioned algorithms in Section 1. ``-'' indicates that the property is not analyzed in their paper, while  ``SCR'' stands for subsequence convergence rate. ``Single/Two'' indicates that the theoretical analysis focuses on a single-variable case, but the method is applied to two-variable problems in their work; this interpretation similarly applies to  ``Multi/Two'' and ``Two/Multi''. }
	\label{summary_table}
	\centering
	\begin{tabular}{c|c|c|c|c|c}
		\hline
		Algorithms & $L$-smooth & Blocks & Extrapolation & SCR & Global under KL \\\hline
		PALM \cite{BolteST14} & Yes & Multi & No & - & Yes \\\hline
		iPALM \cite{PockS16} & Yes & Multi  & Yes & - & Yes \\\hline
		GiPALM \cite{GaoCH20} & Yes & Two & Yes & - & Yes \\\hline
		NiPALM \cite{WangH23a} & Yes & Two & Yes & - & Yes \\\hline
		BPG \cite{BolteSTV18First} & No & Single & No & $\mathcal{O}(1/K)$ & Yes \\\hline
		iBPG \cite{MukkamalaOPS20} & No & Single/Two & Yes &  $\mathcal{O}(1/K)$ & Yes \\\hline
		
		BPDCA(e) \cite{TakahashiFT22} & No & Single & Yes &  $\mathcal{O}(1/K)$ & Yes \\\hline
		DCAe \cite{PhaA24} & No & Single/Two & Yes & - & Yes \\\hline
		BMME \cite{HienPGAP22} & No & Multi/Two & Yes & - & Yes \\\hline
		This paper & No & Two/Multi & Yes &  $\mathcal{O}(1/K)$ & Yes \\\hline
	\end{tabular}
\end{table*}

\textbf{Contribution.}  In this paper, we establish a general model \eqref{NMD-T-Obj} to connect the non-negative sparsity and low-rank properties through the ReLU activation function. Additionally, we present a novel and efficient algorithm tailored to address this model. The key contributions are outlined as follows.
\begin{itemize} 
	\item[(i)] \textbf{Model:} 
	We first propose a novel nonlinear matrix decomposition framework incorporating a comprehensive regularization term designed to simultaneously promote structures such as low-rankness, sparsity, and non-negativity in the resulting factors. 
	\item[(ii)] \textbf{Algorithm:}  Combining the ideas of alternating minimization and Bregman algorithms, we propose an alternating accelerated partial Bregman (AAPB) algorithm (Algorithm \ref{AAPB}) for a general two-block nonconvex nonsmooth problem \eqref{two-general} which includes \eqref{NMD-T-Obj} as a special case. We then apply the AAPB algorithm to solve the optimization problem \eqref{NMD-T-Obj}, see Algorithm \ref{NMD_AAPB} for details which denotes NMD-AAPB. Under the Kurdyka-{\L}ojasiewicz (K\L) property and some suitable conditions, the algorithm's sublinear convergence rate and global convergence are also established. 
	\item[(iii)] \textbf{Efficiency:}  
	We present the closed-form solutions of the $(U, V)$- subproblem for several cases of $H_{1}(U)$ and $H_{2}(V)$, which can be updated simultaneously while preserving the  $L$-smooth adaptable property (Definition \ref{L-smad})   for any $L\ge 1$. Numerical experiments on graph regularized clustering and sparse NMF basis compression demonstrate the effectiveness of the proposed model and algorithm. 
\end{itemize}

The rest of this paper is organized as follows. 
Section \ref{preliminaries} provides some relevant definitions and results which will be used in this paper. Section \ref{algorithms} provides an in-depth exposition of the proposed algorithm (Algorithm \ref{AAPB} and its special case Algorithm \ref{NMD_AAPB}).  We establish the sublinear convergence and global convergence results of the proposed algorithm and present the closed-form solution of the proposed algorithm under different regularizations in Sections \ref{convergence_result} and \ref{closed-form}, respectively. Section \ref{numerical_experiments} utilizes synthetic and real datasets to demonstrate the stability of model \eqref{NMD-T-Obj} and the efficacy of the proposed algorithm, respectively. Finally, we draw conclusions in Section \ref{conclusion}.

\section{Preliminaries}\label{preliminaries}
In this section, we provide a concise summary of pivotal definitions and associated results that are important to the discussions in this paper.
\begin{definition} \label{strong_convex}
	\cite{Beck17} 
	A function $f: \mathbb{R}^d \rightarrow \mathbb{R}$ is $\sigma$-strongly convex if
	\[
	f(y)\ge f(x)+\langle \nabla f(x),y-x\rangle +\frac{\sigma}{2}\|y-x\|^{2}
	\]
	for any $x,y\in\mathbb{R}^{d}$.
	
\end{definition}

\begin{definition}\label{proximal-operator}
	\cite{Beck17} The proximal operator  (also called the
	proximal mapping) associated with a (typically convex) function $f: \mathbb{R}^d \rightarrow \mathbb{R}$, is defined, for every $x \in\mathbb{R}^d$ with $\lambda>0$, as the following minimizer,
	\[
	\text{prox}_{\lambda f}(x):=\underset{y\in\mathbb{R}^{d}}{\arg\min} f(y) +\frac{1}{2\lambda}\|y-x\|^{2}.
	\]
\end{definition}
\begin{definition}\label{soft_def}
	\cite{Donoho95} For any $x\in\mathbb{R}^{d}$,
	\begin{align*}
		S_{\tau}(x) &=\underset{y\in\mathbb{R}^{d}}{\arg\min}\{\tau\|y\|_{1}+\frac{1}{2}\|y-x\|^{2}\}=\max\{|y|-\tau,0\}\text{sign}(y)
	\end{align*}
	with the absolute value understood to be componentwise.
\end{definition}
\begin{definition}\label{hard_def}
	\cite{LussT13} Given $x\in\mathbb{R}^{d}$, without loss of generality we assume that $|x_{1}|\ge|x_{2}|\ge\dots\ge|x_{d}|$, then the hard-thresholding operator $H_{s}(x)$ is given by
	\[
	H_{s}(x)=\underset{y\in\mathbb{R}^{d}}{\arg\min}\{\|y-x\|^{2}:\|y\|_{0}\le s\}=\left\{\begin{matrix}
		\begin{aligned}
			x_{i},\quad & \text{if }i\le s,\\
			0,\quad &\text{otherwise}.
		\end{aligned}
	\end{matrix}\right.
	\]
\end{definition}

\begin{definition}
	\cite{rockafellar1998variational} Let 
	$f:\mathbb{R}^d\rightarrow\mathbb{R}\cup\{+\infty\}$  
	be proper and lower semicontinuous. The basic limiting-subdifferential of $f$ at $x\in\mbox{dom}\, f$, denoted by $\partial f(x)$, is defined by
	\begin{align*}
		\partial f(x):=\{&g\in\mathbb{R}^d:\exists x_{k} \rightarrow x, f(x_{k})\rightarrow f(x), g_k\in \hat{\partial}f(x_k)\text{ with } g_k\rightarrow g \}.
	\end{align*}
	where $\hat{\partial}f(\bar{x})$ denotes the Fr\'{e}chet subdifferential of $f$ at $\bar{x}\in\mbox{dom}\,f$, which is the set of all $g\in\mathbb{R}^d$ satisfying
	\[
	\liminf_{y\neq\bar{x},y\rightarrow\bar{x}}\frac{f(y)-f(\bar{x})-\langle g,y-\bar{x}\rangle}{\|x-\bar{x}\|}\ge0.
	\]
\end{definition}

\begin{definition}(\cite{AttouchBRS10} Kurdyka-{\L}ojasiewicz (K\L) inequality)\label{KL_ineq}: Let $f : \mathbb{R}^d \rightarrow  \mathbb{R}\cup \{+\infty\} $ be a proper lower semicontinuous function. For $-\infty<\eta_{1}<\eta_{2}\le+\infty$, set
	\[
	[\eta_{1}<f<\eta_{2}]=\{ x\in\mathbb{R}^d: \eta_{1}<f(x)<\eta_{2}  \}.
	\]
	We say that function $f$ has the K\L~property at $x^{*}\in\text{dom } \partial f$ if there exist $\eta\in(0, +\infty]$, a neighborhood	$U$ of $x^{*}$, and a continuous concave function $\phi : [0, \eta) \rightarrow \mathbb{R}_{+}$, such that
	\begin{itemize}
		\item[(1)] $\phi(0)=0$;
		\item[(2)] $\phi$ is $\mbox{C}^{1}$ on $(0,\eta)$ and continuous at 0;
		\item[(3)] $\phi'(s)>0$, $\forall s\in(0,\eta)$;
		\item[(4)] $\forall x\in U\cap [f(x^{*})<f<f(x^{*})+\eta]$, the following K\L~inequality holds
		\[
		\phi^{'}(f(x)-f(x^{*}))d(0,\partial f(x))\ge1,
		\]
	\end{itemize}
	where $d(x,\Omega) = \inf_{y \in\Omega}\|x-y\|^{2}$ for any compact set $\Omega$.
\end{definition}
\begin{definition} (\cite{attouch2013convergence} K\L~function)\label{KL_fun}: We denote $\Phi_{\eta}$ as the set of functions which satisfy (1)-(3) in Definition \ref{KL_ineq}. If $f$ satisfies the K\L~property at each point of dom $\partial f$, then $f$ is called a K\L~function.
\end{definition}

For a matrix $X\in\mathbb{R}^{m\times n}$, we denote its Frobenius norm by $\|X\|_{F}:=\sqrt{\sum_{i,j}x_{ij}^{2}}$, 
the $l_{1}$ norm $\|X\|_{1}:=\sum_{i=1}^{m}\sum_{j=1}^{n}|x_{ij}|$, the $l_{0}$ norm  $\|X\|_{0}$ indicates the number of non-zero elements in matrix $X$. For a matrix $Y\in\mathbb{R}^{m\times m}$, we denote its trace as $\text{Tr}(Y):=\sum_{i=1}^{m}Y_{ii}$.

\subsection{Bregman proximal gradient method} \label{BPG-intro}


\begin{definition}\label{def:kernel}
	(\cite{AuslenderT06,BolteSTV18First} Kernel-generating distance). Let $C$ be a nonempty, convex, and open subset of $\mathbb{R}^{d}$. Associated with $C$ {and its closure $\overline{C}$}, a function $\psi:\mathbb{R}^{d} \rightarrow (-\infty, +\infty]$ is called a kernel generating distance if it satisfies the following conditions:
	\begin{itemize}
		\item[(i)] $\psi$ is proper, lower semicontinuous, and convex, with $\text{dom } \psi$ $\subset$ $\overline{C}$ and  $\text{dom } \partial \psi$  $=C$.
		\item[(ii)]  $\psi$ is $C^{1}$ on $\text{int dom } \psi \equiv C$.
	\end{itemize}
	The class of kernel-generating distances is denoted by $\mathcal{G}(C)$. Given $\psi\in\mathcal{G}(C)$, we can define the proximity measure $D_{\psi} : \text{dom } \psi \times \text{int dom }\psi \rightarrow \mathbb{R}_{+}$ by
	\begin{eqnarray}
		D_{\psi}(x, y) := \psi (x) - \psi (y) - \langle \nabla\psi(y), x - y\rangle.
	\end{eqnarray}
	
	The proximity measures $D_{\psi}$ is also called the Bregman distance \cite{Bregman67The}. It measures the proximity of $x$ and $y$. 
\end{definition}
Indeed, $\psi$ is convex if and only if $D_{\psi}(x,y) \ge 0$ for any $x\in\text{dom }\psi$ and $y\in\text{int dom }\psi$.

We examine the following optimization problem to introduce the Bregman proximal gradient (BPG) method, 
\[
\min_{x\in \bar{C}}\,f(x)+h(x),
\]
where $f$ is a continuously differentiable (possibly nonconvex) function that may not have a globally Lipschitz continuous gradient, $h$ is an extended-valued function (possibly nonconvex), $C$ is a nonempty, convex, open subset in $\mathbb{R}^{d}$.

The BPG mapping defined in \cite{BolteSTV18First} is formulated as follows 
\begin{align*}
	x^{k+1}\in &\underset{x\in\bar{C}}{\text{argmin}}\left\{ h(x)+\langle \nabla f(x^{k}), x-x^{k}\rangle +\frac{1}{\lambda}D_{\psi}(x,x^{k}) \right\}\\
	=&\underset{x\in\mathbb{R}^{d}}{\text{argmin}}\left\{ h(x)+\langle \nabla f(x^{k}), x-x^{k}\rangle +\frac{1}{\lambda}D_{\psi}(x,x^{k}) \right\},
\end{align*}
where $\lambda>0$ is the step size. When $\psi=\frac{1}{2}\|\cdot\|^{2}$, this method simplifies to the classic proximal gradient descent algorithm \cite{Beck17}.  

The key to the convergence analysis is the $L$-smooth adaptable property shown as follows.
\begin{definition}\label{L-smad}
	(\cite{BolteSTV18First} $L$-smooth adaptable ($L$-smad)) Given $\psi\in\mathcal{G}(C)$, let $f:\mathcal{X}\rightarrow(-\infty,+\infty]$ be a proper and lower semi-continuous function with $\mathrm{dom}\,\psi\subset\mathrm{dom}\,f$, which is continuously differentiable on $C$. We say $(f, \psi)$ is $L$-smad  on $C$ if there exists $L>0$ such that for any $x,y\in C$,
	\begin{eqnarray}
		|f(x)-f(y)-\langle\nabla f(y),x-y\rangle|\le L D_{\psi}(x,y).\label{L_upper}
	\end{eqnarray}
	If $\psi(\cdot)=\frac{1}{2}\|\cdot\|^{2}$, then it reduces to the   $L$-smooth  \cite{Nesterov18}.
\end{definition}
\section{Algorithm}\label{algorithms}

In this section, we present  an efficient algorithm to solve the optimization problem \eqref{NMD-T-Obj}. Firstly, we consider a more general two-block optimization framework that naturally encompasses the model \eqref{NMD-T-Obj} as a specific instance, i.e., 
\begin{eqnarray}
	\min_{Y\in C_{1},W\in C_{2}}\, \Phi(Y,W):=F(Y, W)+G(W)+H(Y), \label{two-general}
\end{eqnarray}
where $C_{1},C_{2}$ are closed convex sets, $F(Y,W)$ is a smooth but nonconvex function, and $G(W)$ and $H(Y)$ are two nonsmooth (possibly nonconvex) functions. In particular, this non-convex non-smooth optimization problem exhibits the following two key characteristics: Firstly, the $W$-subproblem has a closed-form solution; Secondly, the $Y$-subproblem does not have a  closed-form solution, 
and $\nabla_{Y}F(Y,W)$ is not global Lipschitz continuous for any fixed $W$. 
Building on these two points, and considering the limitations of the existing BPG-type algorithms as discussed earlier, it is essential to design a new, efficient solution algorithm tailored to the specific characteristics of the proposed model.

By leveraging the strengths of alternating minimization \cite{BolteST14, WangHZ24}, 
we first propose the  alternating partial Bregman (APB) method, 
as shown below,
\[
\begin{cases}
	W^{k+1} \in\underset{W\in C_{2}}{\text{argmin}}\, F(Y^{k},W)+G(W),\\
	\begin{aligned}
		Y^{k+1}\in\,\underset{Y\in C_{1}}{\text{argmin}}\, &H(Y)+\langle \nabla_{Y}F(Y^{k},W^{k+1}),Y-Y^{k}\rangle+\frac{1}{\lambda} D_{\psi}(Y,Y^{k}). 
	\end{aligned}
\end{cases}
\]
Here, we utilize the closed-form solution in the $W$-subproblem  and apply the  BPG \cite{BolteSTV18First, mukkamala2022bregman} in the $Y$-subproblem, specifically to tackle the issue that $\nabla_{Y}F(Y,W)$ is not global Lipschitz continuous. 
Here, $D_{\psi}(\cdot,\cdot)$ is the Bregman distance defined by Definition~\ref{def:kernel}.

However, the numerical performance of this framework may not   be optimal.  
To accelerate the convergence and get better numerical performance, we  incorporate the extrapolation technique \cite{MukkamalaOPS20} to the variable $Y$, i.e., 
\begin{align*}
	\bar{Y}^{k}&= Y^{k}+\beta_{k}(Y^{k}-Y^{k-1}).
\end{align*}
and \begin{equation}\label{eq:Yupdate}
	\begin{aligned}
		Y^{k+1}\in\,&\underset{Y\in C_{1}}{\text{argmin}}\, H(Y)+\langle \nabla_{Y}F(\bar{Y}^{k},W^{k+1}),Y-\bar{Y}^{k}\rangle+\frac{1}{\lambda}D_{\psi}(Y,\bar{Y}^{k}).
	\end{aligned}
\end{equation}
It is the extrapolation step to accelerate the convergence and  $\beta_{k}\in[0,1)$. Then the accelerated version of the APB (AAPB) algorithm is established, see Algorithm \ref{AAPB} for details. When $\beta_{k}=0$, then the AAPB algorithm reduces to the APB algorithm.
\begin{algorithm}[!ht]
	\caption{\textbf{AAPB:} Accelerated alternating Partial Bregman algorithm for the optimization problem \eqref{two-general}.}
	\label{AAPB}
	{\bfseries Input:} $0<\lambda\le1/L$, $K$, and kernel-generating distance $\psi$. \\
	{\bfseries Initialization:} $Y^{0}=Y^{-1}$,   
	\begin{algorithmic}[1] 
		\For {$k=0,1,\dots K$} 
		\State Update $W^{k+1}\in\underset{W\in C_{2}}{\text{argmin}}\, F(Y^{k},W)+G(W)$.
		\State Compute $\bar{Y}^{k}= Y^{k}+\beta_{k}(Y^{k}-Y^{k-1})\in C_{1}$ with $\beta_{k}\in[0,1)$.
		\State Update
		\[
		\begin{aligned}
			Y^{k+1}\in\,&\underset{Y\in C_{1}}{\text{argmin}}\, H(Y)+\langle \nabla_{Y}F(\bar{Y}^{k},W^{k+1}),Y-\bar{Y}^{k}\rangle+\frac{1}{\lambda}D_{\psi}(Y,\bar{Y}^{k}).
		\end{aligned}
		\]
		\EndFor
	\end{algorithmic} 
	{\bfseries Output:}  $Y^{k+1}$.
\end{algorithm}


\subsection{AAPB algorithm for \eqref{NMD-T-Obj}}
In this subsection, we apply the AAPB algorithm (Algorithm \ref{AAPB}) to solve the optimization problem \eqref{NMD-T-Obj}. For a comprehensive understanding, please refer to Algorithm \ref{NMD_AAPB} for details.
\begin{algorithm}[!ht]
	\caption{\textbf{NMD-AAPB:} The AAPB algorithm for the optimization problem \eqref{NMD-T-Obj}}
	\label{NMD_AAPB}
	{\bfseries Input:} $M$, $r$, $0<\lambda\le1/L$, $I_{+}$, $I_{0}$, $K$, and kernel-generating distance $\psi$. \\
	{\bfseries Initialization:} $U^{0}=U^{-1}$,  $V^{0}=V^{-1}$, $X^{0}=U^{0}V^{0}$, and set $W_{i,j}=M_{i,j}$ for $(i,j)\in I_{+}$.
	\begin{algorithmic}[1] 
		\For {$k=0,1,\dots K$} 
		\State Compute $W_{i,j}^{k+1}=\min(0,X_{i,j}^{k})$ for $(i,j)
		\in I_{0}$.
		\State Compute $\bar{U}^{k}= U^{k}+\beta_{k}(U^{k}-U^{k-1})$ and $\bar{V}^{k}= V^{k}+\beta_{k}(V^{k}-V^{k-1})$, where $\beta_{k}\in[0,1)$.
		\State Update
		\begin{align*}
			P^{k}:=&\lambda\nabla_{U}F(\bar{U}^{k},\bar{V}^{k},W^{k+1})-\nabla_{U}\psi(\bar{U}^{k},\bar{V}^{k}),\\
			Q^{k}:=&\lambda\nabla_{V}F(\bar{U}^{k},\bar{V}^{k},W^{k+1})-\nabla_{V}\psi(\bar{U}^{k},\bar{V}^{k}).
		\end{align*}
		\begin{eqnarray}
			\begin{aligned}
				(U^{k+1},V^{k+1})\in\,&\text{argmin}\{ \lambda H_{1}(U)+\lambda H_{2}(V)\\
				&+\langle P^{k},U\rangle +\langle Q^{k},V\rangle +\psi(U,V)\}.
			\end{aligned}\label{update_UV_02}
		\end{eqnarray}
		\State Compute $X^{k+1}=U^{k+1}V^{k+1}$.
		\EndFor
	\end{algorithmic} 
	{\bfseries Output:}  $U^{k+1},V^{k+1}$.
\end{algorithm}	



Now we give the details of the proposed algorithm (Algorithm \ref{NMD_AAPB}) designed to solve the optimization problem \eqref{NMD-T-Obj} as follows.
	
	\textbf{$W$-subproblem:} At the $k$-th iteration,
	\begin{align*}
		W^{k+1}=\underset{\max(0,W)=M}{\text{argmin}}\, \frac{1}{2}\|W-X^{k}\|_{F}^{2}, 
	\end{align*}
	where $X^{k}=U^{k}V^{k}$. 
	Then we derive the closed-form solution $W^{k+1}$ as
	\begin{align*}
		W_{i,j}^{k+1}=\begin{cases}
			M_{i,j},\quad &\text{if } (i,j)\in I_{+},\\
			\min(0,X_{i,j}^{k}),\quad &\text{if } (i,j)\in I_{0}.
		\end{cases}
	\end{align*}
	Here, $I_0$ and $I_+$ are defined by \eqref{I_0} and \eqref{I_1}, respectively.
	
	\textbf{$(U,V)$-subproblem}  For the general formulation of $H_{1}(U)$ and $ H_{2}(V)$, the subproblem 
	\[
	(U^{k+1},V^{k+1})\in\underset{U,V}{\text{argmin}} \ F(U,V,W^{k+1})+H_{1}(U)+H_{2}(V)
	\]
	may not have a closed-form solution in general. 
	Inspired by the works in  \cite{BolteSTV18First, MukkamalaO19}, we propose the application of the Bregman proximal gradient method to the  $(U,V)$-subproblem. Denote $F_{k}:=F(\bar{U}^{k},\bar{V}^{k},W^{k+1})$, thus we have
	\begin{align*}
		(U^{k+1},V^{k+1})
		\in&\underset{U,V}{\text{argmin}}\   H_{1}(U)+H_{2}(V)+\langle \nabla_{U} F_{k},U-\bar{U}^{k}\rangle \\
		&+\langle \nabla_{V} F_{k},V-\bar{V}^{k}\rangle +\frac{1}{\lambda}D_{\psi}((U,V),(\bar{U}^{k},\bar{V}^{k}))\\
		=&\underset{U,V}{\text{argmin}}\, H_{1}(U)+H_{2}(V)+\langle \nabla_{U} F_{k},U-\bar{U}^{k}\rangle \\
		&+\langle \nabla_{V} F_{k},V-\bar{V}^{k}\rangle +\frac{1}{\lambda}( \psi(U,V)- \psi(\bar{U}^{k},\bar{V}^{k})
		\\&-\langle \nabla_{U}\psi(\bar{U}^{k},\bar{V}^{k}),U-\bar{U}^{k}\rangle-\langle \nabla_{V}\psi(\bar{U}^{k},\bar{V}^{k}),V-\bar{V}^{k}\rangle )\\
		=&\underset{U,V}{\text{argmin}}\, \lambda H_{1}(U)+\lambda H_{2}(V)+ \langle P^{k},U\rangle +\langle Q^{k},V\rangle +\psi(U,V),
	\end{align*}
	where $\lambda$ is the step size, and
	\begin{align*}
		P^{k}:=&\lambda\nabla_{U}F(\bar{U}^{k},\bar{V}^{k},W^{k+1})-\nabla_{U}\psi(\bar{U}^{k},\bar{V}^{k}),\\
		Q^{k}:=&\lambda\nabla_{V}F(\bar{U}^{k},\bar{V}^{k},W^{k+1})-\nabla_{V}\psi(\bar{U}^{k},\bar{V}^{k}).    
	\end{align*}
	


	\begin{remark}
		Algorithm \ref{NMD_AAPB} presents two distinct advantages:
		\begin{itemize}
			\item[(i)] As stated in Proposition \ref{pro_Lpsi_smad} in Subsection \ref{closed-form}, for the $(U, V)$-subproblem, any $L\ge1$ satisfies the $L$-smad property. This eliminates the practical necessity for estimation.  
			\item[(ii)] The $(U, V)$-subproblem can be updated simultaneously in the proposed algorithm, which significantly reduces computational time in numerical experiments. This advantage is particularly notable when dealing with large-scale data. 
		\end{itemize}
	\end{remark}
	
	\section{Convergence analysis}\label{convergence_result}
	In this section, we present the sublinear convergence and global convergence results of the AAPB algorithm (Algorithm \ref{AAPB}) for solving \eqref{two-general} as follows. Prior to presenting these results, we outline the necessary assumptions. It is important to note that, within the context of this paper $C=\mathbb{R}^{n}$ in Definition \ref{def:kernel}. 
	
	\begin{assumption} \label{assumption_01}
		Assume the following conditions hold:
		\begin{enumerate}
			\item[(i)] The kernel generating distance $\psi$ given by Definition \ref{def:kernel}, is $\sigma$-strongly convex.
			\item[(ii)] $F$ is a continuously differentiable function and the  functions pair $(F(\cdot,  W), \psi)$ is $L$-smooth adaptable {(see Definition \ref{L-smad})}.
			\item[(iii)] $G, H$ are two proper, lower semicontinuous functions. 
			\item[(iv)] $\Phi^{*}:=\inf_{Y,W} \Phi(Y,W)>-\infty$.
			\item[(v)] There exists $\alpha\in\mathbb{R}$ such that $H(\cdot)-\frac{\alpha}{2}\|\cdot\|_{F}^{2}$ is convex. 
		\end{enumerate}
	\end{assumption}
	If $\beta_{k}=0$ in Algorithm \ref{AAPB}, 
	Assumption \ref{assumption_01}~(v) can be removed, and the convergence results also hold.
	\begin{assumption}\label{assumption_02}
		We assume the parameter $\beta_{k}\in[0,1)$ in Algorithm \ref{AAPB} satisfies the following inequality,
		\[
		D_{\psi}(Y^{k},\bar{Y}^{k})\le \frac{\delta-\varepsilon}{1+L\lambda}D_{\psi}(Y^{k-1},Y^{k}),
		\]
		where $1>\delta>\varepsilon>0$, {$\lambda$ denotes the step size in the proposed algorithms, while $L$ is the parameter defined in Definition \ref{L-smad}.} 
	\end{assumption} 
	
	A remark regarding Assumption \ref{assumption_02} is provided as follows. 
	\begin{remark} \label{beta_remark}
		When $\psi=\frac{1}{2}\|\cdot\|^{2}_{F}$, it follows from Assumption \ref{assumption_02} that
		\[
		\beta_{k}\le \sqrt{\frac{\delta-\varepsilon}{1+L\lambda}}.
		\]
		In other words, if $\delta-\varepsilon\approx1$, one could choose the extrapolation parameter $\beta_{k}\approx1/\sqrt{2}$. In general, the closed-form expression for $\beta_{k}$ is difficult to obtain in each iteration. One possible approach is to apply the backtracking line-search strategy \cite{MukkamalaOPS20} to find a proper $\beta_k$ such that the inequality in Assumption \ref{assumption_02} holds. However, this technique may be time-consuming. 
		For simplicity, we set this parameter as a fixed constant.  
	\end{remark}
	
	We further demonstrate the convergence of the proposed algorithm under the K\L~property, as presented in the following two theorems. 
	\begin{theorem} \label{subsequence_convergence}
		(Subsequence convergence of Algorithm \ref{AAPB})  
		Assume   Assumptions \ref{assumption_01} and \ref{assumption_02} hold, and
		$0<\lambda\le1/L$. Let $\{Y^{k}\}_{k\in\mathbb{N}}$ be the sequence generated by the NMD-AAPB algorithm. Then the following statements hold.
		\begin{itemize}
			\item[(i)] We have
			\[
			\sum_{k=1}^{+\infty}D_{\psi}(Y^{k-1},Y^{k})<+\infty,
			\]
			which indicates the sequence $D_{\psi}(Y^{k-1},Y^{k})$ converges to zero with $k\rightarrow +\infty$.  
			\item[(ii)] In addition, {we have}
			\[
			\underset{1\le k\le K}{\min}D_{\psi}(Y^{k-1},Y^{k})= \mathcal{O}(1/K).
			\]
		\end{itemize}
		
	\end{theorem}
	\begin{proof}
		From the $W$-subproblem in Algorithm \ref{AAPB}, we obtain
		\begin{eqnarray}
			F(Y^{k},W^{k+1}) +G(W^{k+1}) \le F(Y^{k},W^{k}) +G(W^{k}). \label{G_inequality}
		\end{eqnarray}
		Given the convexity of $H(Y)-\frac{\alpha}{2}\|Y\|^{2}_{F}$ as stated  in Assumption \ref{assumption_01}, we can derive  that
		\begin{align*}
			H(Y^{k})-\frac{\alpha}{2}\|Y^{k}\|_{F}^{2}\ge &H(Y^{k+1})-\frac{\alpha}{2}\|Y^{k+1}\|_{F}^{2}+\left\langle \xi^{k+1}-\alpha Y^{k+1},Y^{k}-Y^{k+1}\right\rangle,
		\end{align*}
		where $\xi^{k+1}\in\partial H(Y^{k+1})$. It can be reformulated equivalently as follows, 
		\begin{align*}
			H(Y^{k+1})+\frac{\alpha}{2}\|Y^{k+1}-Y^{k}\|_{F}^{2} +\left\langle \xi^{k+1},Y^{k}-Y^{k+1}\right\rangle \le H(Y^{k}).
		\end{align*}
		From the first-order optimality condition of the $Y$-subproblem in Algorithm \ref{AAPB}, we can derive that  
		\begin{align*}
			\xi^{k+1}+\nabla_{Y}F(\bar{Y}^{k},W^{k+1}) +\frac{1}{\lambda}(\nabla\psi(Y^{k+1})-\nabla \psi(\bar{Y}^{k}))=0.
		\end{align*}
		By combining the above two inequalities, we can derive the following result 
		\begin{eqnarray}
			\begin{aligned}
				&H(Y^{k+1})+\frac{\alpha}{2}\|Y^{k+1}-Y^{k}\|_{F}^{2}-\langle\nabla_{Y}F(\bar{Y}^{k},W^{k+1}),Y^{k}-Y^{k+1}\rangle \\
				&+\frac{1}{\lambda}\langle \nabla\psi(\bar{Y}^{k})-\nabla \psi(Y^{k+1}),Y^{k}-Y^{k+1}\rangle \\
				=&H(Y^{k+1})+\frac{\alpha}{2}\|Y^{k+1}-Y^{k}\|_{F}^{2}-\langle\nabla_{Y}F(\bar{Y}^{k},W^{k+1}),Y^{k}-Y^{k+1}\rangle\\ &+\frac{1}{\lambda}(D_{\psi}(Y^{k},Y^{k+1})+D_{\psi}(Y^{k+1},\bar{Y}^{k})-D_{\psi}(Y^{k},\bar{Y}^{k}))\\
				\le& H(Y^{k}),
			\end{aligned}\label{H_ineq_01}
		\end{eqnarray}
		where the last equality follows from the three-point identity. Furthermore, since $F(\cdot, W)$ is an $L$-smad function with respect to $\psi$, we have
		\begin{align*}
			&F(Y^{k+1},W^{k+1})- F(\bar{Y}^{k},W^{k+1})\\
			\le&\langle\nabla_{Y}F(\bar{Y}^{k},W^{k+1}),Y^{k+1}-\bar{Y}^{k}\rangle +LD_{\psi}(Y^{k+1},\bar{Y}^{k}),
		\end{align*}
		and
		\begin{align*}
			&F(\bar{Y}^{k},W^{k+1})+\langle \nabla_{Y}F(\bar{Y}^{k},W^{k+1}),Y^{k}-\bar{Y}^{k}\rangle\\
			\le &F(Y^{k},W^{k+1})+LD_{\psi}(Y^{k},\bar{Y}^{k}).
		\end{align*}
		By combining the above two inequalities, we have 
		\begin{eqnarray}
			\begin{aligned}
				F(Y^{k+1},W^{k+1})
				\le&F(Y^{k},W^{k+1})+\langle \nabla_{Y}F(\bar{Y}^{k},W^{k+1}),Y^{k+1}-Y^{k}\rangle\\
				&+LD_{\psi}(Y^{k},\bar{Y}^{k})+LD_{\psi}(Y^{k+1},\bar{Y}^{k}).
			\end{aligned} \label{F_ineq_01}
		\end{eqnarray}
		By summing inequalities \eqref{G_inequality}, \eqref{H_ineq_01} and \eqref{F_ineq_01} together, we obtain
		\begin{align*}
			&\Phi(Y^{k+1},W^{k+1})\\
			\le&\Phi(Y^{k},W^{k})+(L+\frac{1}{\lambda})D_{\psi}(Y^{k},\bar{Y}^{k}) -\frac{1}{\lambda}D_{\psi}(Y^{k},Y^{k+1})\\
			&+(L-\frac{1}{\lambda})D_{\psi}(Y^{k+1},\bar{Y}^{k})-\frac{\alpha}{2}\|Y^{k+1}-Y^{k}\|_{F}^{2}\\
			\le&\Phi(Y^{k},W^{k})+(L+\frac{1}{\lambda})D_{\psi}(Y^{k},\bar{Y}^{k}) -\frac{1}{\lambda}D_{\psi}(Y^{k},Y^{k+1})-\frac{\alpha}{2}\|Y^{k+1}-Y^{k}\|_{F}^{2},
		\end{align*}
		where the last inequality follows from $\lambda L\le 1$.
		We define, with $k\in\mathbb{N}$, the following Lyapunov sequence 
		\[
		\Theta^{k}:= \lambda(\Phi(Y^{k},W^{k})-\Phi^{*}) +\delta D_{\psi}(Y^{k-1},Y^{k}).
		\]
		Then we can derive that
		\begin{align*}
			&\Theta^{k}-\Theta^{k+1}\\
			=&\lambda\Phi(Y^{k},W^{k}) +\delta D_{\psi}(Y^{k-1},Y^{k})-\lambda\Phi(Y^{k+1},W^{k+1})-\delta D_{\psi}(Y^{k},Y^{k+1})\\
			\ge&\delta D_{\psi}(Y^{k-1},Y^{k})+(1-\delta)D_{\psi}(Y^{k},Y^{k+1})+\frac{\alpha\lambda}{2}\|Y^{k+1}-Y^{k}\|_{F}^{2}\\
			&-(1+L\lambda)D_{\psi}(Y^{k},\bar{Y}^{k})\\
			\ge&(\frac{\alpha\lambda}{2}+\frac{(1-\delta)\sigma}{2})\|Y^{k+1}-Y^{k}\|_{F}^{2}+\delta D_{\psi}(Y^{k-1},Y^{k})-(1+L\lambda)D_{\psi}(Y^{k},\bar{Y}^{k})\\
			\ge&\delta D_{\psi}(Y^{k-1},Y^{k})-(\delta-\epsilon)D_{\psi}(Y^{k-1},Y^{k})\\
			=&\epsilon D_{\psi}(Y^{k-1},Y^{k}),
		\end{align*}
		where the second inequality follows from $1-\delta>0$, the $\sigma$-strong convexity of $\psi$ and $\lambda^{-1}\ge L\ge-\alpha(1-\delta)\sigma$, and the third inequality follows from Assumption \ref{assumption_02}. Then we have
		\begin{align*}
			\sum_{k=1}^{K}D_{\psi}(Y^{k-1},Y^{k})\le \frac{1}{\epsilon}(\Theta^{1}-\Theta^{K+1})\le\frac{1}{\epsilon}\Theta^{1}.
		\end{align*}
		Taking $K\rightarrow+\infty$, we obtain that 
		\[
		\sum_{k=1}^{+\infty}D_{\psi}(Y^{k-1},Y^{k})<+\infty,
		\]
		from which we immediately deduce that the sequence $\{D_{\psi}(Y^{k-1},Y^{k})\}$ converges to zero. 
		
		In addition, we also have 
		\begin{align*}
			K\min_{1\le k\le K}D_{\psi}(Y^{k-1},Y^{k})\le \sum_{k=1}^{K}D_{\psi}(Y^{k-1},Y^{k})\le \frac{1}{\epsilon}\Theta^{1},
		\end{align*}
		which yields the desired result. 
		This completes the proof.
	\end{proof}
	
	The above theorem could only establish the sub-sequential convergence of the AAPB algorithm if the sequence remains in a bounded set.
	We also require the following assumption to analyze the global convergence. It should be emphasized that the local Lipschitz continuity assumption for both $\nabla F(Y, W)$ and $\nabla\psi$ is sufficient in this paper, which is a weaker requirement compared to global Lipschitz continuity.
	\begin{assumption} \label{assumption_03}
		We assume the following two conditions hold:
		\begin{itemize}
			\item[(i)] $\nabla F(Y,W)$  is Lipschitz continuous with
			constant $L_{1} > 0$ on any bounded subset of $\text{dom} F$.
			\item[(ii)] $\nabla\psi$ is Lipschitz continuous with constant $L_{2} > 0$ on any bounded subset of $\text{dom} \psi$.
		\end{itemize}
	\end{assumption}
	
	Based on the above theorem, combined with the K\L~framework \cite{BolteST14}, we demonstrate the whole sequence convergence as follows.
	\begin{theorem}\label{global_convergence}
		(Global convergence of Algorithm \ref{AAPB}) 
		Suppose that Assumptions \ref{assumption_01}, \ref{assumption_02} and \ref{assumption_03} hold, and $0<\eta\le1/L$.  Let $\{Z^{k}:=(Y^{k},W^{k})\}_{k\in\mathbb{N}}$ be a bounded sequence generated by the AAPB algorithm. Assume that the optimization function $\Phi(Y,W)$ is a {semi-algebraic function \cite{attouch2013convergence}} that satisﬁes the K\L~property with exponent $\theta \in [0, 1)$ (see Definition \ref{KL_fun}). Then either the point $Z^{k}$ is a critical point after a ﬁnite number of iterations or the sequence $\{ Z^{k}\}_{k \in\mathbb{N}}$  satisﬁes the ﬁnite length property,
		\[
		\sum_{k=0}^{+\infty}\|Z^{k+1}-Z^{k}\|_{F}<+\infty.
		\] 
	\end{theorem}
	\begin{proof}

		From the optimality condition of $Y$- and $W$-subproblems in Algorithm \ref{AAPB}, we have
		\begin{eqnarray}
			\begin{cases}
				&0\in\partial H(Y^{k+1}) +\nabla_{Y}F(\bar{Y}^{k},W^{k+1}) +\frac{1}{\lambda} (\nabla\psi(Y^{k+1}) - \nabla \psi(\bar{Y}^{k})),\\
				&0\in\nabla_{W} F(Y^{k},W^{k+1}) +\partial G(W^{k+1}).
			\end{cases}
		\end{eqnarray}
		Denote 
		\[
		\begin{aligned}
			A_{1}^{k+1} :=&\nabla_{Y}F(Y^{k+1},W^{k+1})-\nabla_{Y}F(\bar{Y}^{k},W^{k+1})  -\frac{1}{\lambda}(\nabla\psi(Y^{k+1}) - \nabla \psi(\bar{Y}^{k})),\\
			A_{2}^{k+1} := &\nabla_{W}F(Y^{k+1},W^{k+1}) -\nabla_{W}F(Y^{k},W^{k+1}).     
		\end{aligned}
		\]
		We have
		\[
		A_{1}^{k+1} \in \nabla_{Y}F(Y^{k+1},W^{k+1}) +\partial H(Y^{k+1}),\quad A_{2}^{k+1}\in \nabla_{W}F(Y^{k+1},W^{k+1}) +\partial G(W^{k+1}).
		\]
		Consequently, $A^{k+1}=(A_{1}^{k+1},A_{2}^{k+1})\in \partial \Phi(Y^{k+1},W^{k+1})$ and 
		\[
		\begin{aligned}
			&\|A^{k+1}\|_{F}\\
			\le& \|A_{1}^{k+1}\|_{F}+ \|A_{2}^{k+1}\|_{F}\\
			=&\| \nabla_{Y}F(Y^{k+1},W^{k+1})-\nabla_{Y}F(\bar{Y}^{k},W^{k+1})  -\frac{1}{\lambda}(\nabla\psi(Y^{k+1}) - \nabla \psi(\bar{Y}^{k}))\|_{F}\\
			&+ \| \nabla_{W}F(Y^{k+1},W^{k+1})-\nabla_{W}F(Y^{k},W^{k+1})\|_{F}\\
			\le& \| \nabla_{Y}F(Y^{k+1},W^{k+1})-\nabla_{Y}F(\bar{Y}^{k},W^{k+1})\|_{F} + \frac{1}{\lambda}\|\nabla\psi(Y^{k+1}) - \nabla \psi(\bar{Y}^{k})\|_{F}\\
			&+\| \nabla_{W}F(Y^{k+1},W^{k+1})-\nabla_{W}F(Y^{k},W^{k+1})\|_{F}\\
			\le& (L_{1}+\frac{L_{2}}{\lambda})\|Y^{k+1}-\bar{Y}^{k}\|_{F} +L_{1}\|Y^{k+1}-Y^{k}\|_{F}\\
			\le& (2L_{1}+\frac{L_{2}}{\lambda})\|Y^{k+1}-Y^{k}\|_{F}+\beta_{k}(L_{1}+\frac{L_{2}}{\lambda})\|Y^{k}-Y^{k-1}\|_{F},
		\end{aligned}
		\]
		where the third inequality follows from Assumption \ref{assumption_03}, and the last inequality follows from $\bar{Y}^{k}=Y^{k}+\beta_{k}(Y^{k}-Y^{k-1})$ and triangle inequality.

		Furthermore, under the assumption that the function $\Phi(Y, W)$ is a proper, lower semi-continuous, and semi-algebraic function \cite{attouch2013convergence}, it satisfies the K\L~property at every point within its domain  $\text{dom}\Phi$. By incorporating Definition \ref{KL_fun} and Theorem \ref{subsequence_convergence}, we can deduce that the sequence $\{Z^{k}\}$ generated is indeed a Cauchy sequence. The comprehensive proof of this theorem bears resemblance to Theorem 4.5 in \cite{AhookhoshHGP21} and Theorem 4.2 in \cite{HienPGAP22}. Thus the details are omitted for brevity.  
	\end{proof}
	
	It is worth noting that if the fixed step size $\lambda$ in the proposed algorithm (Algorithm \ref{AAPB}) is changed to an adaptive case $\lambda_{k}$, the convergence results of the proposed algorithm still  hold if $\lambda_{k}\le\min\{\lambda_{k-1},1/L\}$.
	
	
	\section{Closed-form solutions} \label{closed-form}
	In this section, we provide a detailed analysis of the closed-form solutions for the $(U,V)$-subproblem under different cases of $H_{1}(U)$ and $H_{2}(V)$.
	
	First, we present the $L$-smad property for the term $F(U,V,W^{k+1})=\frac{1}{2}\|W^{k+1}-UV\|_{F}^{2}$ in the content of matrix decomposition. The following two propositions (Propositions \ref{pro_Lpsi_smad} and \ref{UV_proposition}) are similar to the corresponding results in \cite{MukkamalaO19, mukkamala2022bregman, WangLCH24} with some straightforward modifications. The primary difference between our kernel generating distance and those in \cite{MukkamalaO19, mukkamala2022bregman, WangLCH24} is that theirs is fixed, whereas ours changes with the number of iterations due to the fact that BPG is incorporated in our AAPB framework.  However, the basic proof process remains largely unchanged, so we omit the proof here for simplicity. Interested readers can refer to the details in \cite{MukkamalaO19, mukkamala2022bregman, WangLCH24}.
	
	\begin{proposition} \label{pro_Lpsi_smad}
		At the $k$-th iteration, let $F(U,V,W^{k+1})=\frac{1}{2}\|W^{k+1}-UV\|_{F}^{2}$, $\psi_{1}=\left(\frac{\|U\|^{2}_{F}+\|V\|^{2}_{F}}{2}\right)^{2}$, $\psi_{2}=\frac{\|U\|_{F}^{2}+\|V\|_{F}^{2}}{2}$. Then, for any $L\ge1$, the function $F$ satisfies the $L$-smad property (Definition \ref{L-smad}) with respect to the following kernel generating distance
		\begin{eqnarray}
			\psi(U,V)=3\psi_{1}(U,V)+\|W^{k+1}\|_{F}\psi_{2}(U,V).\label{psi_def_01}
		\end{eqnarray}
	\end{proposition}

	Now, we discuss the closed-form solution for $(U,V)$ in the NMD-AAPB algorithm (i.e., \eqref{update_UV_02} in Algorithm \ref{NMD_AAPB}) for several cases of $H_{1}(U)$ and $H_{2}(V)$. 
	The Bregman distances for $(U,V)$-subproblem under different cases of $H_{1}(U)+H_{2}(V)$ are given in Table \ref{Bregman_diff}. Then the corresponding closed-form solutions are given in the following proposition. 
	\begin{table*}[!ht]
		\fontsize{8}{15}\selectfont
		\caption{The Bregman distances for $(U,V)$-subproblem under different cases of $H_{1}(U)+H_{2}(V)$ at the $k$-th iteration.}
		\label{Bregman_diff}
		\centering
		\begin{tabular}{c|c|c}
			\hline
			\text{Case}&$H_{1}(U)+H_{2}(V)$  &  $ \psi(U,V)$ \\ \hline
			(i)& $0+0$ & $3\psi_{1}(U,V)+\|W^{k+1}\|_{F}\psi_{2}(U,V)$\\\hline
			(ii)& $\frac{\eta_{1}}{2}\|U\|_{F}^{2}+\frac{\eta_{2}}{2}\|V\|_{F}^{2}$ & $3\psi_{1}(U,V)+\|W^{k+1}\|_{F}\psi_{2}(U,V)$\\\hline
			(iii)& $\eta_{1}\|U\|_{1}+\eta_{2} \|V\|_{1}$ & $3\psi_{1}(U,V)+\|W^{k+1}\|_{F}\psi_{2}(U,V)$\\\hline
			(iv)&$\frac{\mu_{0}}{2}\text{Tr}(U^{T}LU)+\frac{\eta_{1}}{2}\|U\|_{F}^{2}+\frac{\eta_{2}}{2}\|V\|_{F}^{2}$   & $3\psi_{1}(U,V)+(\|W^{k+1}\|_{F}+\mu_{0}\|L\|_{F})\psi_{2}(U,V)$ \\ \hline 
			(v)&$\eta_1\|U\|_1-\frac{\eta_2}{2}\|U\|_F^2$ & $3\psi_{1}(U,V)+\|W^{k+1}\|_{F}\psi_{2}(U,V)+\frac{\eta_{2}\lambda}{2}\|U\|_{F}^{2}$  \\ \hline
			(vi)&$I_{\|U_{,:}\|_0\le s_1}+I_{\|V_{,:}\|_0\le s_2}$  & $3\psi_{1}(U,V)+\|W^{k+1}\|_{F}\psi_{2}(U,V)$ \\ \hline
			
		\end{tabular}
	\end{table*}

	\begin{proposition} \label{UV_proposition}
		In this proposition, we provide the closed-form solutions of $(U, V)$-subproblem as follows.
		\begin{itemize}
			\item[(i)] $\mathbf{H_{1}(U)=H_{2}(V)=0}$. At the $k$-th iteration,  the closed-form solutions of $(U^{k+1},V^{k+1})$ are given by $U^{k+1}=-tP^{k}, V^{k+1}=-tQ^{k}$ respectively, where $t$ is the non-negative real root of
			\[
			3\left(\|P^{k}\|^{2}_{F}+\|Q^{k}\|_{F}^{2}\right) t^{3}+\|W^{k+1}\|_{F}t-1=0.
			\]
			\item[(ii)] $\mathbf{H_{1}(U)=\frac{\eta_{1}}{2}\|U\|_{F}^{2}}, \mathbf{H_{2}(V)=\frac{\eta_{2}}{2}\|V\|_{F}^{2}}$. At the $k$-th iteration, the closed-form solutions $(U^{k+1},V^{k+1})$ are given by $U^{k+1}=-t_1P^{k}, V^{k+1}=-t_2Q^{k}$ respectively, where $t_1$ and $t_2$ are the  non-negative real root of 
			\[
			3\left(\|P^{k}\|_{F}^{2}+\|Q^{k}\|_{F}^{2}\right) t^{3} +(\|W^{k+1}\|_{F}+\eta_{1})t-1=0
			\]
			and 
			\[
			3\left(\|P^{k}\|_{F}^{2}+\|Q^{k}\|_{F}^{2}\right) t^{3} +(\|W^{k+1}\|_{F}+\eta_{2})t-1=0,
			\]
			respectively. 
			\item[(iii)] $\mathbf{H_{1}(U)=\eta_{1}\|U\|_{1}}, \mathbf{H_{2}(V)=\eta_{2} \|V\|_{1}}$. 
			At the $k$-th iteration,  the update step \eqref{update_UV_02} is given by  $U^{k+1}=t\mathcal{S}_{\eta_{1}\lambda}(-P^{k}), V^{k+1}=t\mathcal{S}_{\eta_{2}\lambda}(-Q^{k})$ respectively,  where $t\ge0$ and satisfies 
			\[
			\begin{aligned}
				3(\|\mathcal{S}_{\eta_{1}\lambda}(-P^{k})\|_{F}^{2}+&\|\mathcal{S}_{\eta_{2}\lambda}(-Q^{k})\|_{F}^{2})t^{3}+\|W^{k+1}\|_{F}t-1=0,
			\end{aligned}
			\]
			where $\mathcal{S}_{\lambda}(\cdot)$ is the soft-thresholding operator, see Definition \ref{soft_def} for details.
			\item[(iv)] $\mathbf{H_{1}(\cdot)=\frac{\mu_{0}}{2}\text{Tr}(U^{T}LU)+\frac{\eta_{1}}{2}\|U\|_{F}^{2}}$, $\mathbf{H_{2}(\cdot)=\frac{\eta_{2}}{2} \|V\|_{F}^{2}}$. 
			At the $k$-th iteration, the update step of $(U,V)$ is given by $U^{k+1}=-t_{1}P^{k}$, $V^{k+1}=-t_{2}Q^{k}$ respectively, where $t_{1},t_{2}\ge0$ and satisfies
			\[
			\begin{aligned}
				3(\|P^{k}\|_{F}^{2}+\|Q^{k}\|_{F}^{2})t^{3}+&(\|W^{k+1}\|_{F}+\mu_{0}\|L\|_{F}+\eta_{1})t-1=0
			\end{aligned}
			\]
			and 
			\[
			\begin{aligned}
				3(\|P^{k}\|_{F}^{2}+\|Q^{k}\|_{F}^{2})t^{3}+&(\|W^{k+1}\|_{F}+\mu_{0}\|L\|_{F}+\eta_{2})t-1=0,
			\end{aligned}
			\]
			respectively.
			\item[(v)] $\mathbf{H_{1}(U)=\eta_{1}\|U\|_{1}-\frac{\eta_{2}}{2}\|U\|_{F}^{2}}$, $\mathbf{H_{2}(V)=0}$. At the $k$-th iteration, the update step \eqref{update_UV_02} is given by $U^{k+1}=t\mathcal{S}_{\eta_{1}\lambda}(-P^{k})$, $V^{k+1}=-tQ^{k}$ respectively, where $t\ge0$ and satisfies
			\[
			\begin{aligned}
				3(\|\mathcal{S}_{\eta_{1}\lambda}(-P^{k})\|_{F}^{2}&+\|-Q_{k}\|_{F}^{2}) t^{3}+\|W^{k+1}\|_{F}t-1=0.
			\end{aligned}
			\]
			\item[(vi)] $\mathbf{H_{1}(U)=I_{\|U\|_{0}\le s_{1}}}$, $\mathbf{H_{2}(V)=I_{\|V\|_{0}\le s_{2}}}$. 
			At the $k$-th iteration,  the update step \eqref{update_UV_02} is given by $U^{k+1}=t\mathcal{H}_{s_{1}}(-P^{k}), V^{k+1}=t\mathcal{H}_{s_{2}}(-Q^{k})$ respectively,  where $t\ge0$ and satisfies
			\[
			\begin{aligned}
				3(\|\mathcal{H}_{s_{1}}(-Q^{k})\|_{F}^{2}+&\|\mathcal{H}_{s_{2}}(-Q^{k})\|_{F}^{2})t^{3}+\|W^{k+1}\|_{F}t-1=0,
			\end{aligned}
			\]
			where $\mathcal{H}_{s}(\cdot)$ is the hard-thresholding operator, see Definition \ref{hard_def} for details.
		\end{itemize}  
	\end{proposition}
	\begin{remark}
		The $t$-subproblem in the above proposition requires solving a one-dimensional cubic equation as follows,
		\[at^3+ct-1=0,\]
		where $a,c>0$. The unique positive solution \cite{Fan89} would be $t=\frac{-\sqrt[3]{\alpha_{1}}-\sqrt[3]{\alpha_{2}}}{3a}$ with $\alpha_{1}=3a(\frac{-9a+\sqrt{81a^{2}+12ac^{3}}}{2})$ and $\alpha_{2}=3a(\frac{-9a-\sqrt{81a^{2}+12ac^{3}}}{2})$, respectively.
	\end{remark}
	
	
	\section{Numerical experiments}\label{numerical_experiments}
	This section presents numerical experiments evaluating the real-world performance of our proposed algorithm under varying regularization strengths. Experiments are implemented in MATLAB   using a system equipped with a 3.4 GHz Intel Core i7-14700KF processor and 64 GB RAM. 
	
	The algorithm terminates if one of the following three conditions is met:
	\begin{itemize}
		\item[(a)] The maximum run time ($\max_{T}$) is reached. 
		\item[(b)] The maximum number of iterations ($\max_{K}$) is reached.
		\item[(c)] All algorithms for ReLU-based models of the synthetic dataset terminate when
		\[
		\text{Tol}:=\frac{\|M-\max(0,UV)\|_{F}}{\|M\|_{F}}\le 10^{-4}. 
		\]
		All algorithms for ReLU-based models of the real dataset will use the following quantity
		\[
		\text{Tol}:=\frac{\|M-\max(0,X)\|_{F}}{\|M\|_{F}}\le\text{e}_{\min},
		\]
		where $X$ is the solution, and $\text{e}_{\min}$ is the smallest relative error obtained by all algorithms within the allotted time.
	\end{itemize}
	For simplicity, we set the extrapolation parameter\footnote{Here, $\beta_k=0.6$ is taken simply to meet the conditions in Remark \ref{beta_remark}. A larger value or an adaptive approach can be employed, as numerical experiments demonstrate that convergence is still guaranteed and potentially improved numerical results can be obtained.} $\beta_{k}=0.6$ for the NMD-AAPB algorithm, and set the step size $\lambda\le1/L$ with $L\ge1$ from Proposition \ref{pro_Lpsi_smad} for all the proposed algorithms.  This section presents the analysis of the regularization methods associated with cases (iii), (iv), and (vi) in Table \ref{Bregman_diff}, with each case being investigated separately.  
	
	\subsection{Graph regularized matrix factorization}
	Extensive research in the literature \cite{Cai11GNMF, ShahnazBPP06, AhmedHAD21}  has demonstrated the effectiveness of matrix decomposition techniques for clustering problems, particularly in document and image clustering applications. However, a critical limitation of these existing approaches lies in their reliance on strict nonnegativity constraints $U, V\ge 0$, which may excessively restrict the solution space compared to the more flexible constraint  $UV\ge 0$. This stringent requirement often leads to suboptimal low-rank approximations. To address this limitation, we propose to relax the non-negative constraints and consider the following optimization problem: 
	\begin{eqnarray}
		\begin{aligned}
			\underset{U,V,W}{\min}\, &\ \frac{1}{2}\|W-UV\|_{F}^{2}+\frac{\mu_{0}}{2}\text{Tr}(U^{T}\bar{L}U)+\frac{\eta_{1}}{2}\|U\|_{F}^{2}+\frac{\eta_{2}}{2} \|V\|_{F}^{2},\\
			\text{s.t.}\,\,\, &\ \max(0,W)=M. 
		\end{aligned} \label{GMF_model}
	\end{eqnarray}
	where $\text{Tr}(U^{T}\bar{L}U)$ is the graph-regularization term, $\text{Tr}(\cdot)$ denotes the trace of a matrix, $\bar{L}$ denotes the graph Laplacian matrix,  and $\mu_{0}$, $\eta_{1}$ and $\eta_{2}$ are three positive parameters. We use three datasets\footnote{\url{http://www.cad.zju.edu.cn/home/dengcai/Data/data.html}} \emph{COIL20}, \emph{PIE}, and \emph{TDT2}  (see Table \ref{real_lgl2_details} for details) to illustrate the numerical performance of the proposed algorithm. In this numerical experiment, we let $\mu_{0}=100$ , and $r=20$ for \emph{COIL20}, $r=68$  for \emph{PIE}, $r=100$ for \emph{COIL100}, and $r=30$ for \emph{TDT2}, respectively. 
	
	\begin{table}[h!]
		\begin{center}
			\caption{The details for three datasets. ``sparsity ($\%$)'' stands for the proportion of non-zero elements in the data.}  \label{real_lgl2_details}
			\begin{tabular}{c|c c c c} 
				\hline
				Data & $m$ & $n$ & $r$&sparsity ($\%$) \\\hline
				\emph{PIE} & 2856 &1024 & 68 & 91.47\\
				\emph{COIL20} & 1440 & 1024 &20& 65.61\\
				\emph{TDT2} & 9394 & 36771 & 30 &0.35\\\hline
			\end{tabular}
		\end{center}
	\end{table}
	We compute the clustering labels by applying the $K$-means method\footnote{\url{http://www.cad.zju.edu.cn/home/dengcai/Data/Clustering.html}} to the matrix $U$. The clustering accuracy, defined as the proportion of correctly predicted class labels between the obtained labels and the ground truth labels, is computed following the methodology described in \cite{CaiHH05}. 
	Additionally, we conduct a comprehensive comparison between our proposed ReLU-based graph matrix decomposition model \eqref{GMF_model} and the well-established graph-regularized NMF framework \cite{Cai11GNMF}. Specifically, 
	\begin{eqnarray}
		\begin{aligned}
			\min_{U\in\mathbb{R}^{m\times r}_{+},V\in\mathbb{R}^{r\times n}_{+}}\,\frac{1}{2}\|M-UV\|_{F}^{2}&+\frac{\mu_{0}}{2}\text{Tr}(U^{T}LU)+\frac{\eta_{1}}{2}\|U\|_{F}^{2}+\frac{\eta_{2}}{2} \|V\|_{F}^{2}.
		\end{aligned}\label{GNMF_model}
	\end{eqnarray}
	We employ two advanced optimization approaches: the BPG method \cite{BolteSTV18First} and its enhanced variant, BPG with extrapolation (BPGE) \cite{MukkamalaOPS20, WangLCH24} to address this optimization problem. 
	
	For both clustering models, we set the parameters as follows:  $\mu_{0}=100$, $\eta_{1}=\eta_{2}=0.1$. To maintain consistency across all datasets, we implemented a uniform stopping criterion of  $\text{max}_{k}=100$ for all algorithms.  The clustering performance was evaluated using accuracy metrics, with detailed methodological references available in \cite{CaiHH05}.  
	The comparative results of all algorithms across three datasets are systematically presented in Table  \ref{GNMG_clustering}.  This table provides a clear overview of each algorithm's performance in the clustering tasks. 
	\begin{table}[h!]
		\begin{center}
			\caption{Comparison of clustering accuracy $(\%)$ on three datasets for models \eqref{GMF_model} and \eqref{GNMF_model} under compared algorithms. }   \label{GNMG_clustering}
			\begin{tabular}{c|c | c  c |  c c} 
				\hline
				\multirow{2}*{Data} & & \multicolumn{2}{c|}{model \eqref{GNMF_model}} & \multicolumn{2}{c}{model \eqref{GMF_model}}\\
				&K-means& BPG& BPGE & NMD-APB &NMD-AAPB \\\hline
				\emph{PIE} & 54.41& 84.09& 88.78&84.92& \textbf{88.80}\\
				\emph{COIL20} &73.86& 86.74& 87.62&87.60& \textbf{89.04}\\
				\emph{TDT2}& 67.43 & 75.43& 81.78 &77.69& \textbf{84.24}\\\hline
			\end{tabular}
		\end{center}
	\end{table}
	
	As evidenced by the  results presented in Table \ref{GNMG_clustering}, both the proposed ReLU-based model \eqref{GMF_model} and the classical graph-regularized NMF approach \eqref{GNMF_model} significantly outperform the direct application of the K-means algorithm to the raw data matrix $M$. Furthermore, our proposed ReLU-based model \eqref{GMF_model} demonstrates superior clustering performance over the classical graph-regularized NMF approach \eqref{GNMF_model}. This empirical validation not only confirms the effectiveness of our proposed model but also highlights the computational efficiency of the accompanying algorithm. Notably, the accelerated version of our algorithm consistently outperforms its non-accelerated counterpart across all experimental scenarios. 
	
	A particularly interesting observation from Table \ref{GNMG_clustering} reveals a strong correlation between data sparsity and performance improvement: the sparser the dataset, the more significant the performance gains achieved by our ReLU-based decomposition model compared to traditional approaches. This trend suggests that the ReLU function's inherent properties are particularly well-suited for handling sparse data structures.
	
	Based on these comprehensive numerical results, we conclude that for applications involving non-negative sparse datasets, the ReLU-based optimization model represents a promising alternative that can potentially deliver enhanced numerical performance. This finding opens new possibilities for applying our approach to various real-world scenarios where sparse data representations are prevalent. 
	
	\subsection{Compression of sparse NMF Basis}
	We now investigate an additional application of ReLU-NMD: the compression of sparse non-negative dictionaries  \cite{LeeS99}. To demonstrate this application, we employ two standard facial image datasets \footnote{http://www.cad.zju.edu.cn/home/dengcai/Data/FaceData.html}: the \emph{ORL} dataset with   $m=4096$ features and $n=400$ samples, and the \emph{YaleB} dataset with   $m=1024$ features and $n=2414$ samples. Through the NMF decomposition framework $M \approx \tilde{U}\tilde{V}$,  where both factor matrices $\tilde{U}$ and $\tilde{V}$ maintain non-negativity constraints, we can effectively extract sparse facial features that are represented by the columns of the basis matrix $\tilde{U}$.

	\begin{figure}[!ht]
		\setlength\tabcolsep{2pt}
		\centering
		\begin{tabular}{c}
			\includegraphics[width=0.95\columnwidth]{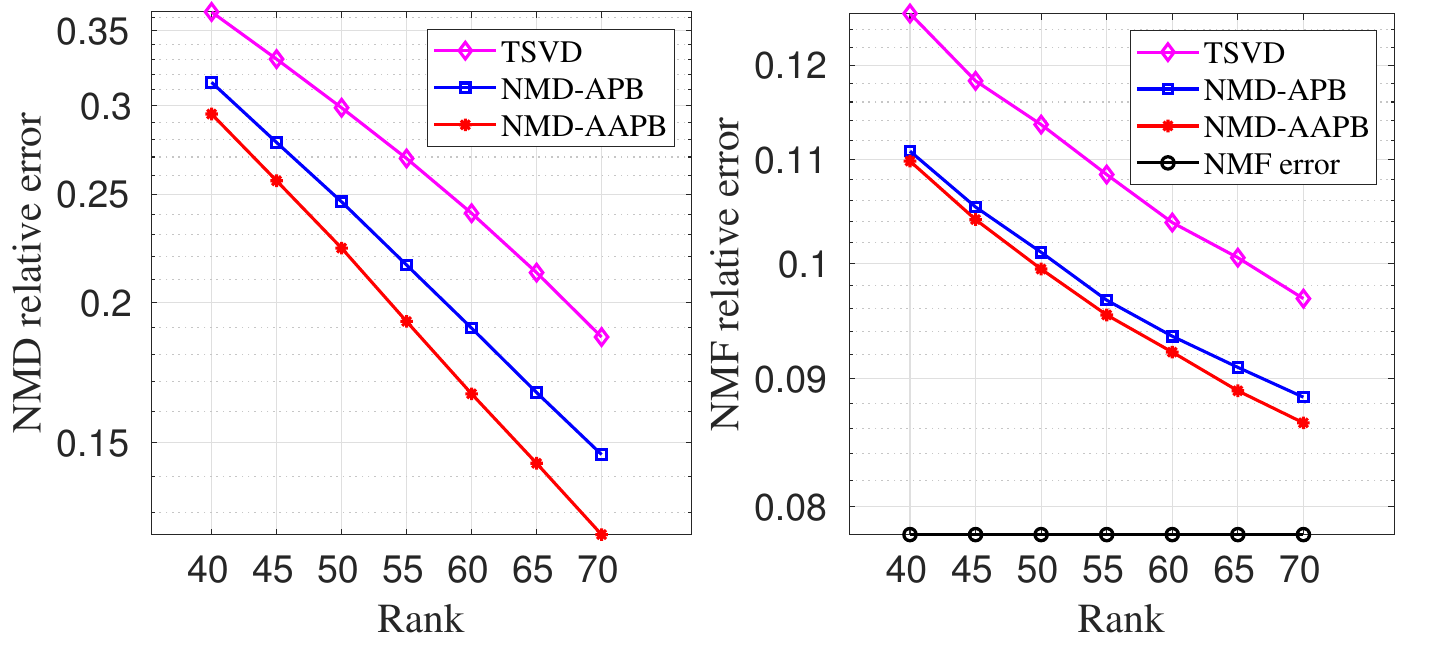}\\
			(a) \emph{ORL}. Left: Error on $U$. Right: Error on $M$.\\
			\includegraphics[width=0.95\columnwidth]{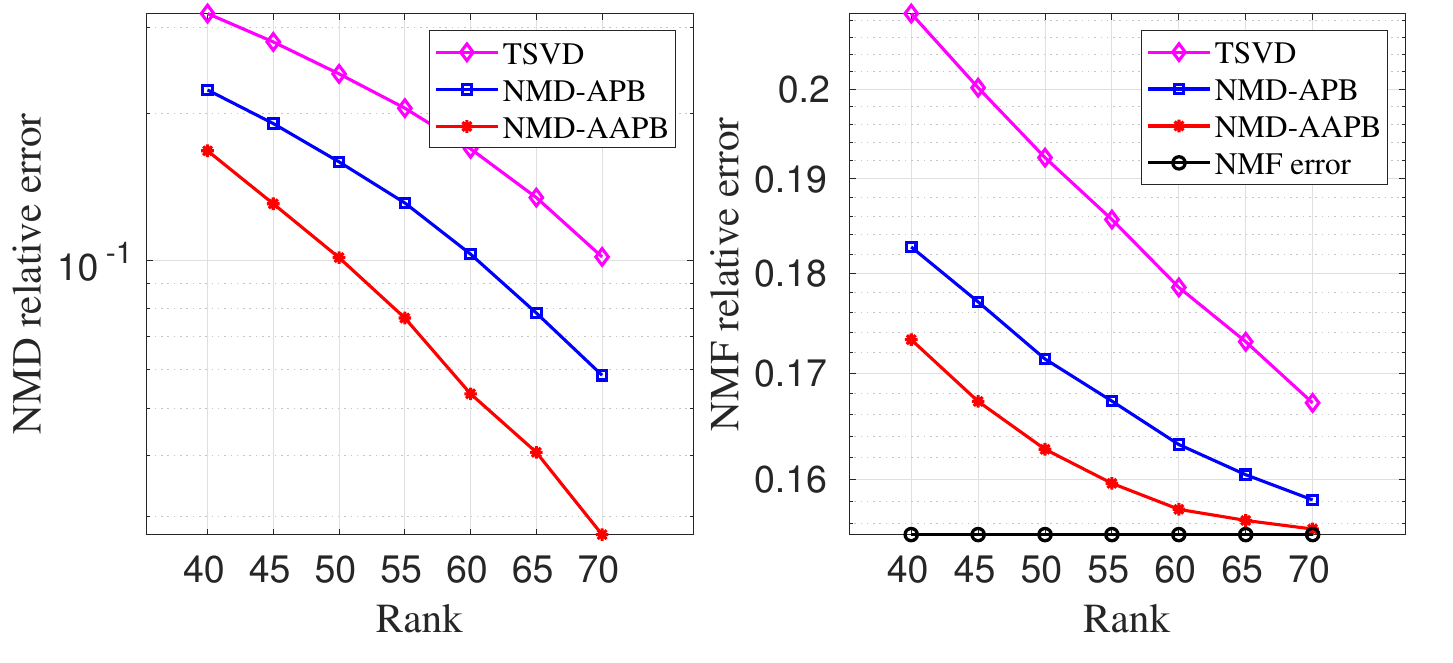}\\
			(b) \emph{YaleB}. Left: Error on $U$. Right: Error on $M$.\\
		\end{tabular}
		\caption{Numerical experiments of real-world datasets for solving \eqref{NMF_com}. } 
		\label{real_l0_com}
	\end{figure}
	
	\begin{figure}
		\setlength\tabcolsep{3pt}
		\centering
		\begin{tabular}{ccccccc}
			\includegraphics[width=0.44\textwidth]{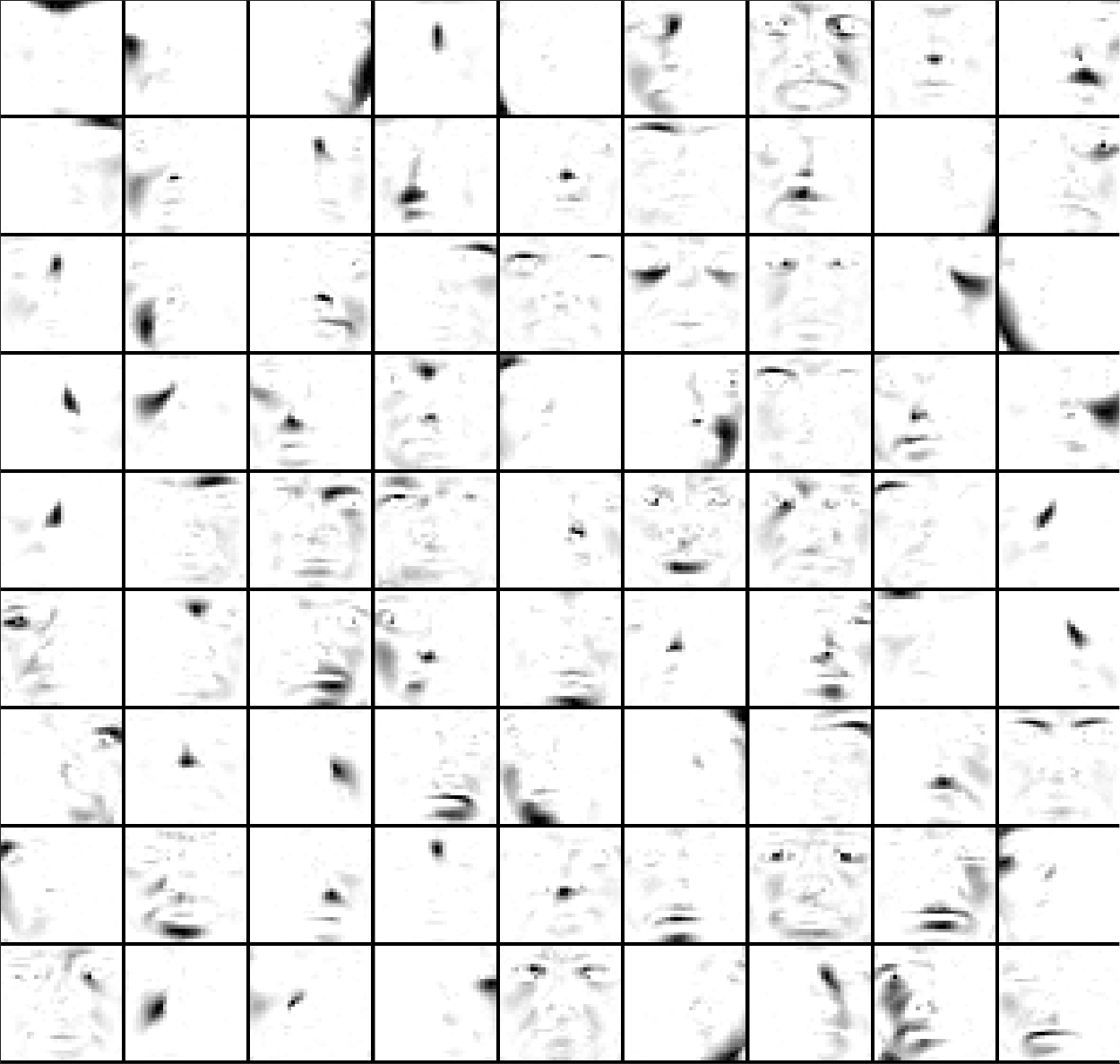}&
			\includegraphics[width=0.44\textwidth]{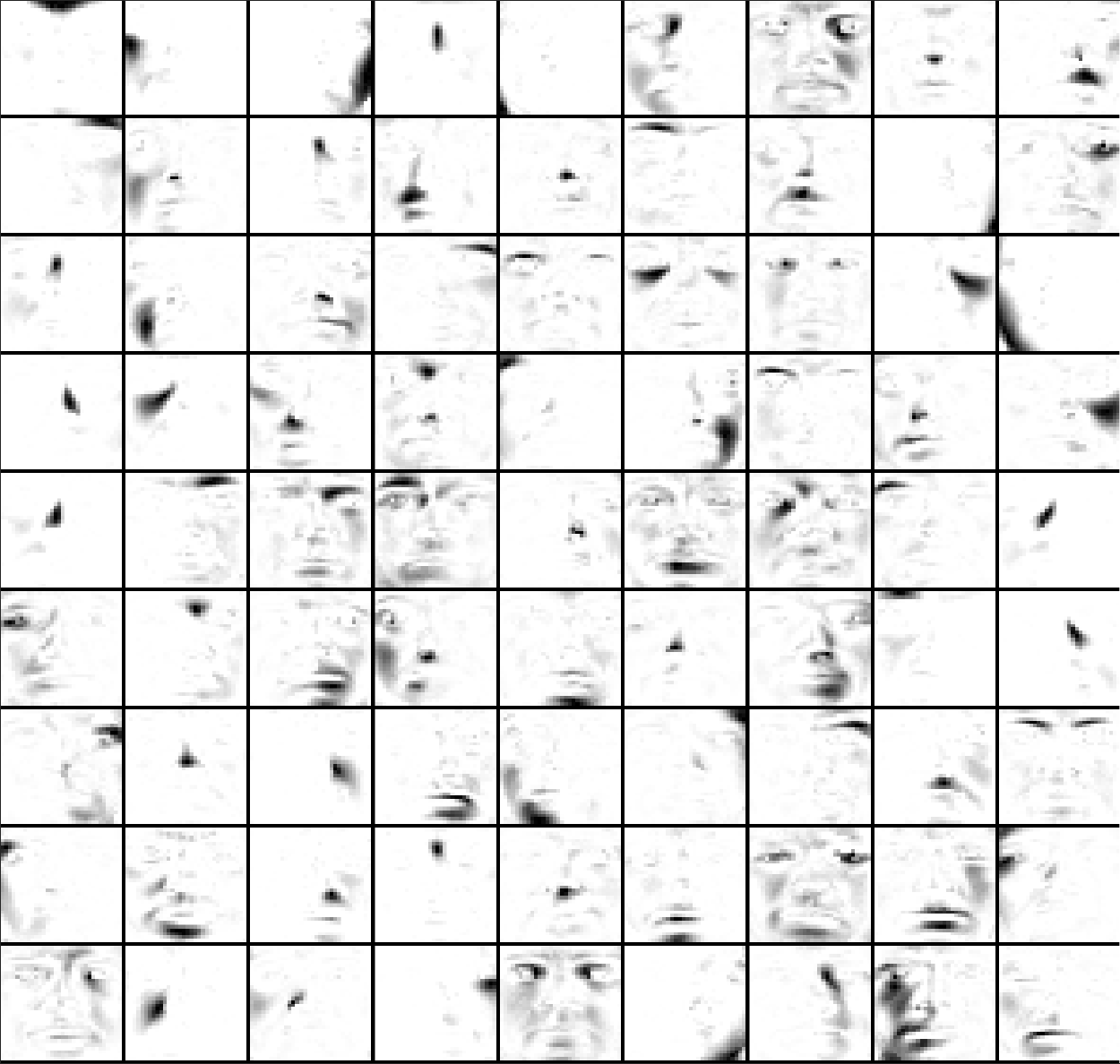}\\
			(a) Original $\tilde{U}$, $r=81$. &(b) TSVD, {$\text{Tol}_{\text{NMF}}=0.19$.} \\
			\includegraphics[width=0.44\textwidth]{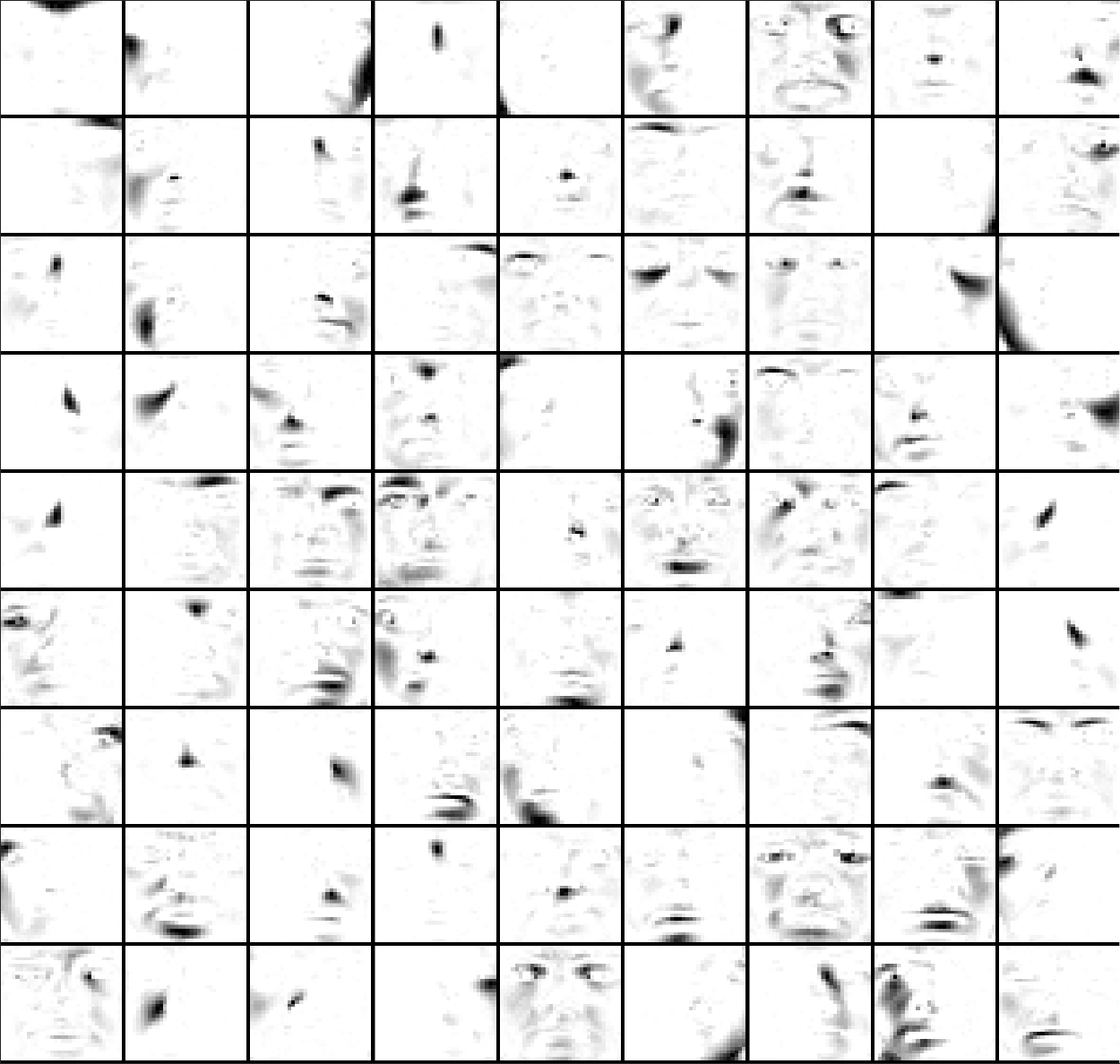}&
			\includegraphics[width=0.44\textwidth]{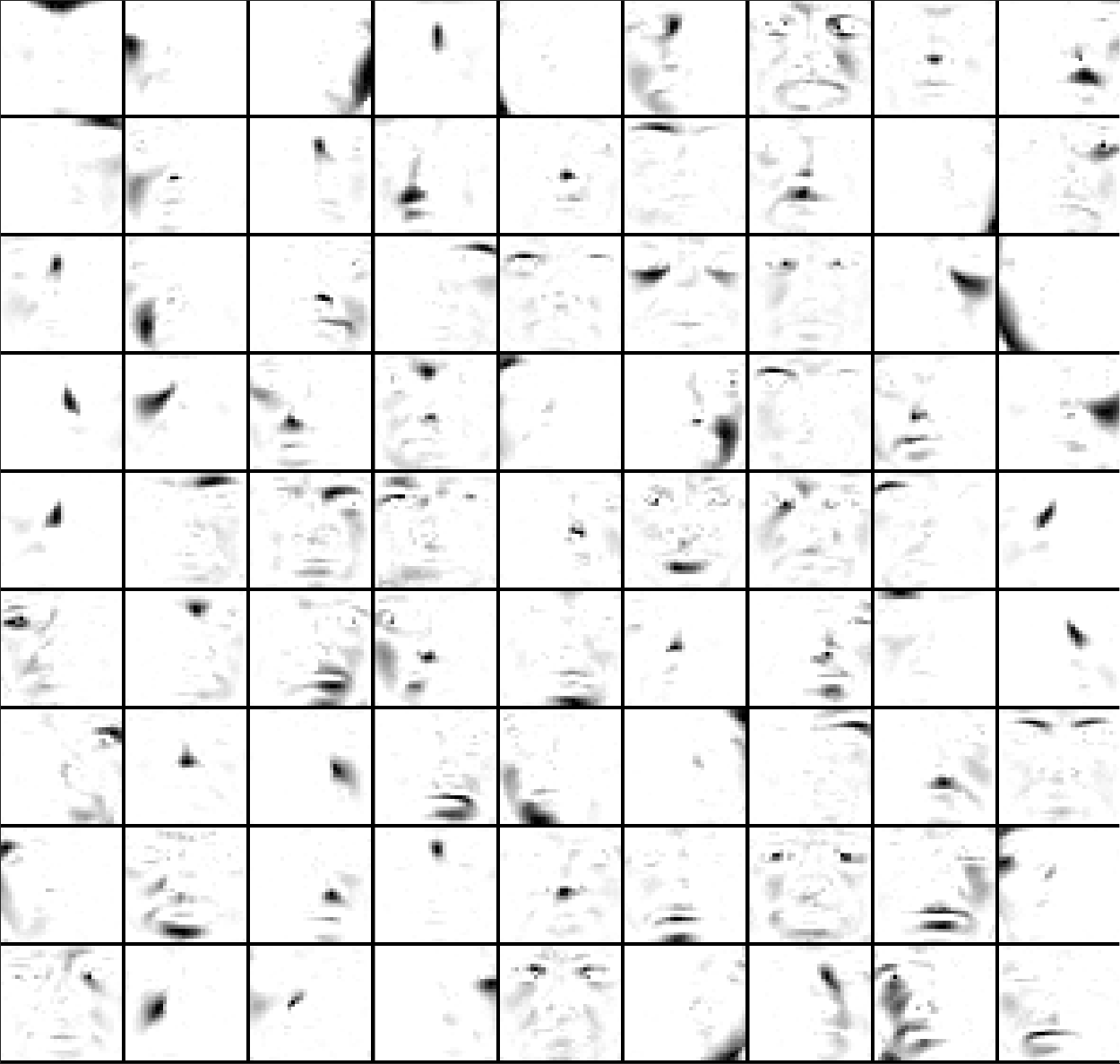}\\
			(c) NMD-APB, {$\text{Tol}_{\text{NMF}}=0.17$.}&(d) NMD-AAPB,   {$\text{Tol}_{\text{NMF}}=0.16$.}
		\end{tabular}
		\caption{Original factor $\tilde{U}$ of NMF for \emph{YaleB} dataset, with rank-$r=81$ and low-rank reconstruction by TSVD \cite{BoutsidisM14}, NMD-APB and NMD-AAPB with fixed rank-$r=55$.}
		\label{yaleb_res}
	\end{figure}
	
	\begin{figure}
		\setlength\tabcolsep{3pt}
		\centering
		\begin{tabular}{ccccccc}
			\includegraphics[width=0.22\textwidth]{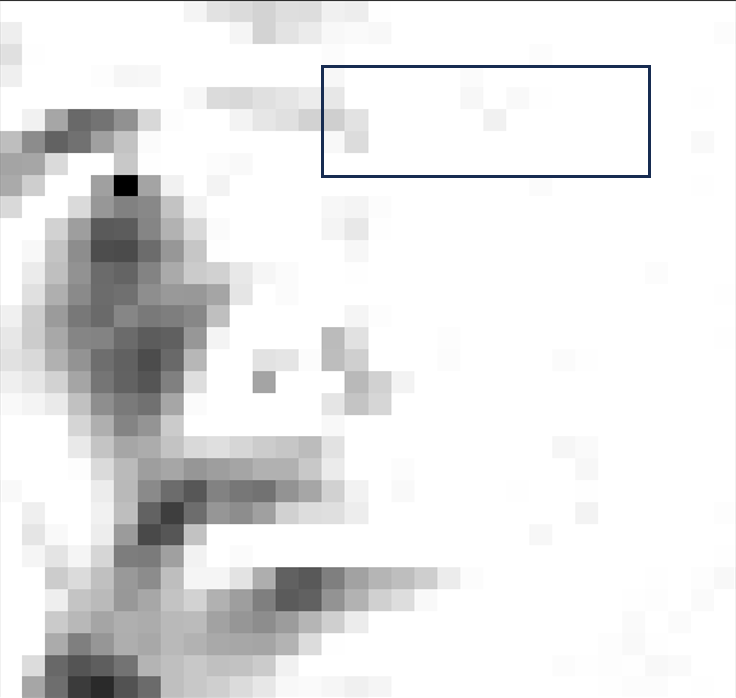}&
			\includegraphics[width=0.22\textwidth]{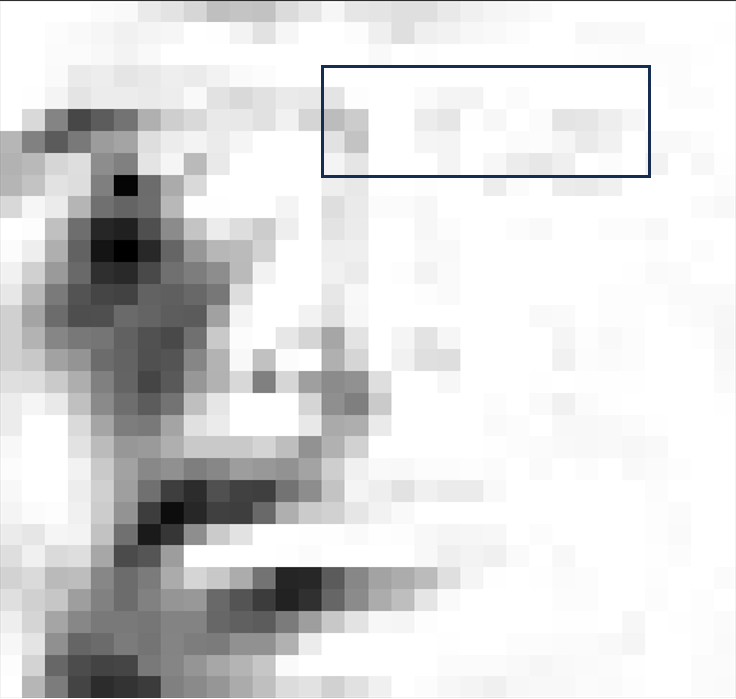}&
			\includegraphics[width=0.22\textwidth]{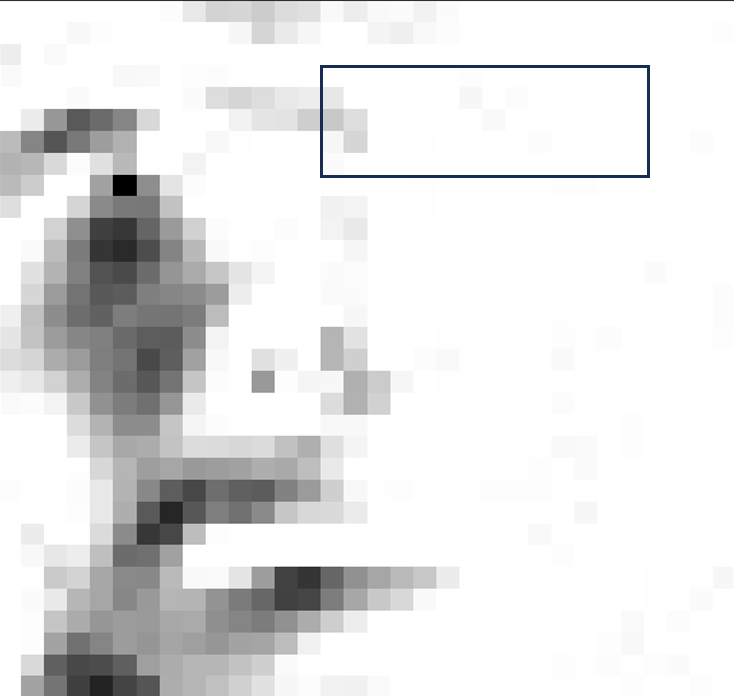}&
			\includegraphics[width=0.22\textwidth]{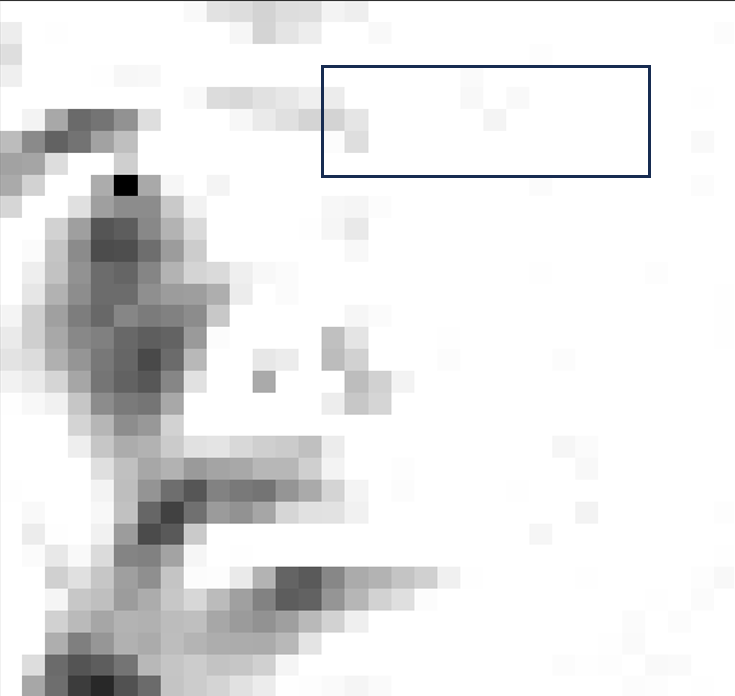}\\
			\includegraphics[width=0.22\textwidth]{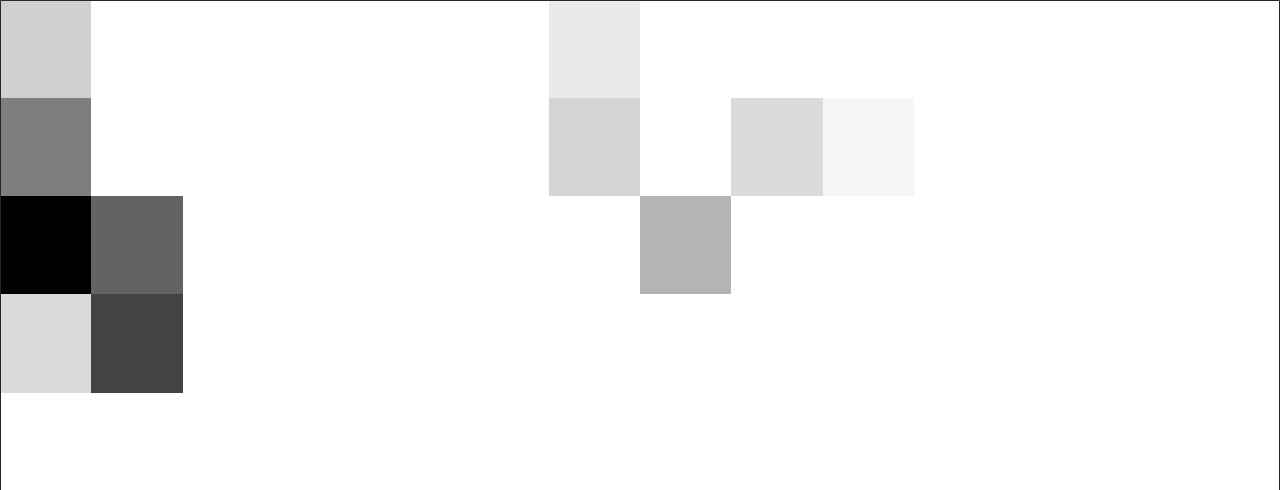}&
			\includegraphics[width=0.22\textwidth]{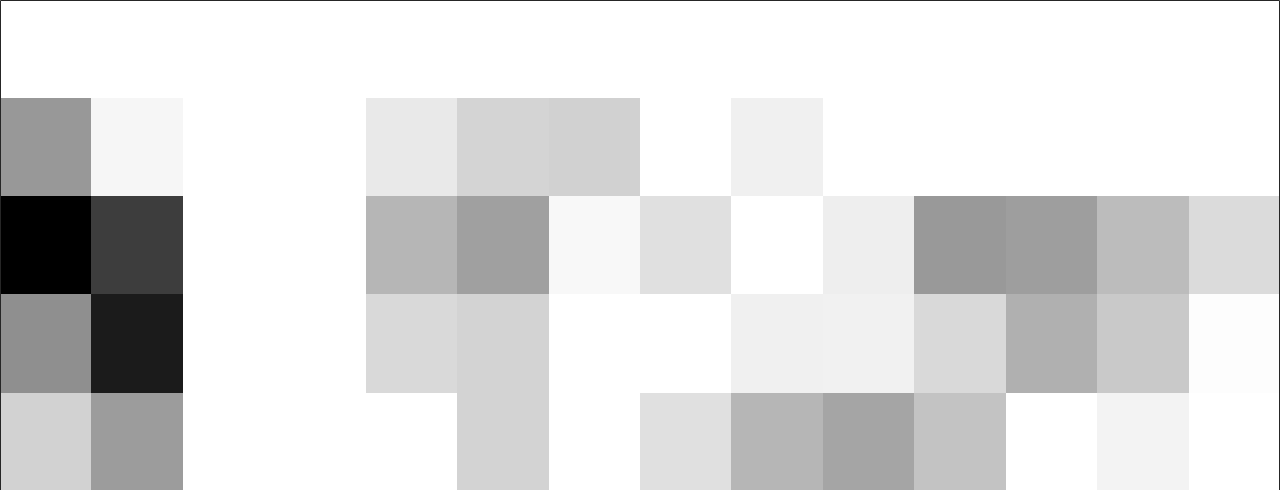}&
			\includegraphics[width=0.22\textwidth]{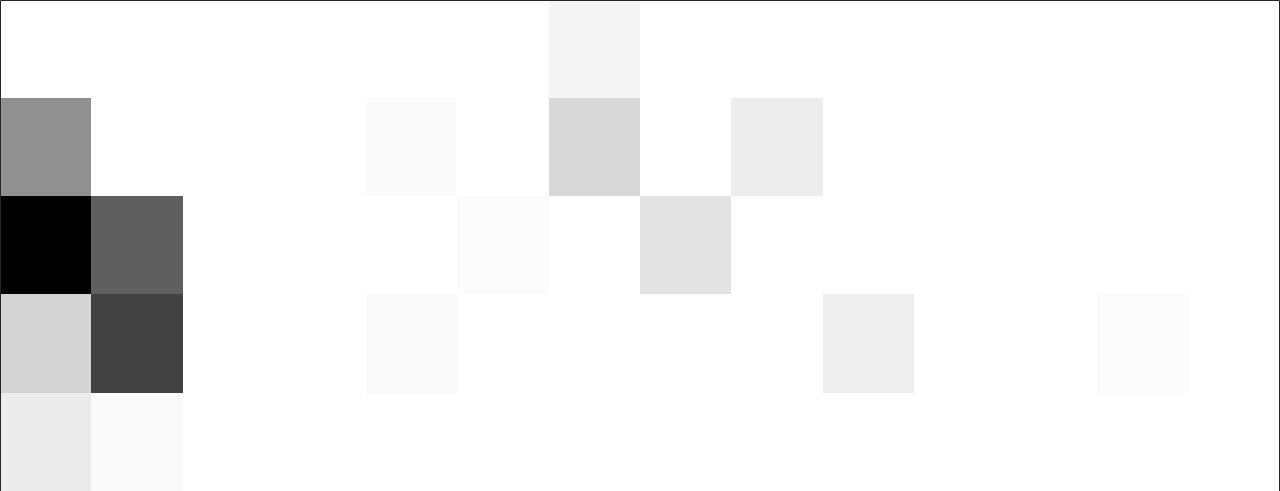}&
			\includegraphics[width=0.22\textwidth]{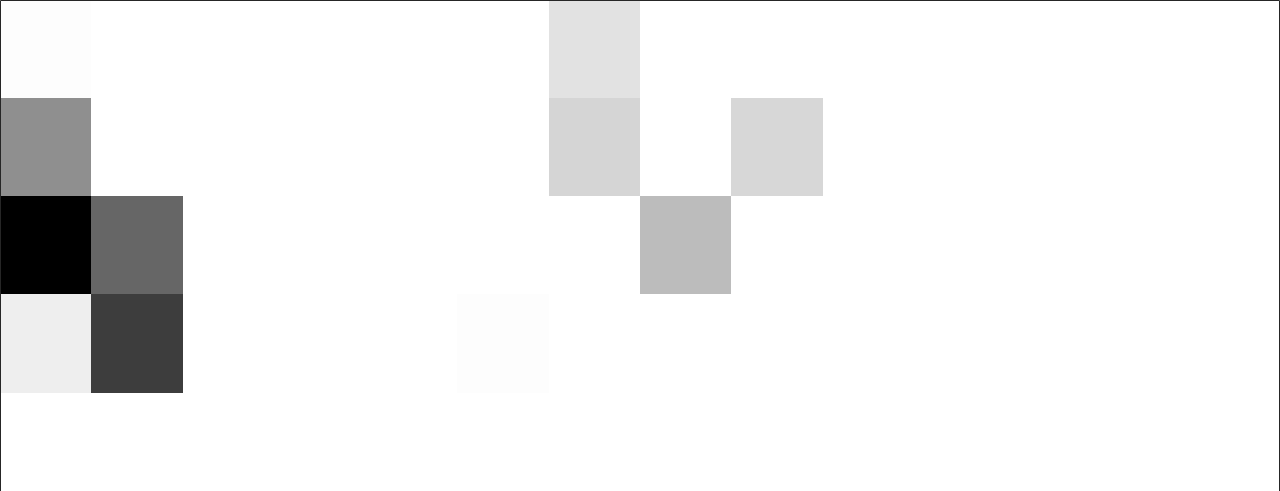}\\
		\end{tabular}
		\caption{The $80$-th basis of the factor $U$ in Figure~\ref{yaleb_res}. Left to right: Original, TSVD \cite{BoutsidisM14}, NMD-APB, NMD-AAPB.}
		\label{yaleb_res_80th}
	\end{figure}
	
	Specifically, we implement a rank-$100$ NMF decomposition\footnote{\url{https://gitlab.com/ngillis/nmfbook/}} on the \emph{ORL} dataset, producing a sparse non-negative matrix $\tilde{U} \in \mathbb{R}^{4096 \times 100}$. Similarly, for the \emph{YaleB} dataset, we perform a rank-$81$ NMF decomposition, resulting in  $\tilde{U} \in \mathbb{R}^{1024 \times 81}$.  Based on these decompositions, we further compress $\tilde U\in\mathbb R^{m\times r}$ into a thin basis matrix $U\in\mathbb R^{m\times r'}$, where $r'<r$ by the following optimization model:
	\begin{eqnarray}
		\begin{aligned}
			\underset{U,V,W}{\min}\, &\ \frac{1}{2}\|W-UV\|_{F}^{2}+\sum_{i=1}^{r}I_{\|{(V^{T})}_{i}\|_{0}\le s_{2}},\\
			\text{s.t.}\,\,\, &\ \max(0,W)=\tilde{U}.
		\end{aligned} \label{NMF_com}
	\end{eqnarray}
	Here, $U\in\mathbb R^{m\times r'}$, $V\in\mathbb R^{r'\times r}$, and $V_i$ is the $i$-th column of $V$. For simplicity, we set {$s_{2}=\lfloor r/1.2\rfloor$}  to demonstrate the numerical performance of model \eqref{NMF_com} for the proposed algorithms.  
	The performance of the compressed NMF is evaluated through the NMD relative error  and  the NMF relative error metric, defined as: 
	\begin{eqnarray}
		\text{Tol}_{\text{NMF}}:=\min_{\hat{V}\ge0}\frac{\|M-\max(0,UV)\hat{V}\|_{F}}{\|M\|_{F}}, 
	\end{eqnarray}
	The results are visually presented in Figure~\ref{real_l0_com}, where Figure~\ref{real_l0_com}~(a) displays the compression of a $4096$-by-$100$ NMF basis $U$ from the \emph{ORL} dataset, and Figure~\ref{real_l0_com}~(b) shows the compression of a $1024$-by-$81$ NMF basis $U$ from the \emph{YaleB} dataset.
	
	The numerical results presented in Figure~\ref{real_l0_com} clearly illustrate that the proposed NMD-APB and NMD-AAPB outperform the baseline methods. This superiority can be attributed to the fact that the ReLU-based model is more adept at capturing the underlying structure of the dataset, and furthermore, the availability of a closed-form solution for the $(U, V)$-subproblem facilitates faster convergence. 
	
	Additionally, Figure~\ref{yaleb_res} presents an instance of a rank $r = 55$ reconstruction derived from the original rank $r = 81$ NMF factor for the \emph{YaleB} dataset. We also show the details of the $80$-th basis of $U$ in  Figure~\ref{yaleb_res_80th}. This example further confirms the superior accuracy of the proposed model and algorithms in approximating the dataset.

\subsection{$H_{1}(U)=\eta_{1}\|U\|_{1}, H_{2}(V)=\eta_{2} \|V\|_{1}$}\label{11-regular}
In this subsection, we focus on the scenario where $H_{1}(U)=\eta_{1}\|U\|_{1}$, $H_{2}(V)=\eta_{2} \|V\|_{1}$ in \eqref{NMD-T-Obj}, which reformulated \eqref{NMD-T-Obj} equivalently as the following optimization problem
\begin{eqnarray}
	\begin{aligned}
		\underset{U,V,W}{\min}\, &\ \frac{1}{2}\|W-UV\|_{F}^{2}+\eta_{1}\|U\|_{1}+\eta_{2} \|V\|_{1},\\
		\text{s.t.}\,\,\, &\ \max(0,W)=M. 
	\end{aligned} \label{L1-L1-model}
\end{eqnarray}
For the above optimization problem, algorithms such as A-NMD, 3B-NMD, and NMD-TM are not suitable because both the $U$- and $V$-subproblems lack closed-form solutions. 
We compare our proposed methods with a PALM-type algorithm \cite{BolteST14} that  incorporates  a partial linearization technique \cite{WangHZ24}. This variant is denoted as PPALM, given that the $W$-subproblem consistently admits a closed-form solution. Additionally, we take into account its inertial counterpart, referred to as iPPALM, which is a special case of iPAMPL \cite{WangHZ24}. For a comprehensive understanding, please refer to Algorithm \ref{iPPALM} in Appendix \ref{ippalm_sec} for further details.

We first consider the synthetic datasets generated by 
\[
M=\max(0,UV),
\]
where $U\in\mathbb{R}^{m\times r^{*}}$ and $V\in\mathbb R^{r^*\times n}$ are generated by command ``$\mathrm{U=randn(m,r^{*})}$, $\mathrm{U(abs(U)<1)=0}$'' and ``$\mathrm{V=randn(r^{*},n)}$, $\mathrm{V(abs(V)<1)=0}$'' in MATLAB. 
We set $\eta_{1}=0.01$, $\eta_{2}=0.015$, $\max_{T}=20\,s$, and $\max_{K}=1000$. 

The numerical experiments presented in   Figure~\ref{syn_l1l1} demonstrate  that  NMD-APB and NMD-AAPB algorithms proposed in this paper converge faster than PPALM and iPPALM. 
One possible reason that  NMD-APB and NMD-AAPB   adopt the closed-form solution of the $(U,V)$-subproblem, while PPALM and iPPALM update the $U$- and $V$-subproblems  by the proximal gradient descent approach separately.  
The acceleration technique yields superior numerical results compared to the vanilla versions. Additionally, a higher estimated rank corresponds to a better fit to the original data.

\begin{figure}[!ht]
	\setlength\tabcolsep{2pt}
	\centering
	\begin{tabular}{c}
		\includegraphics[width=0.95\columnwidth]{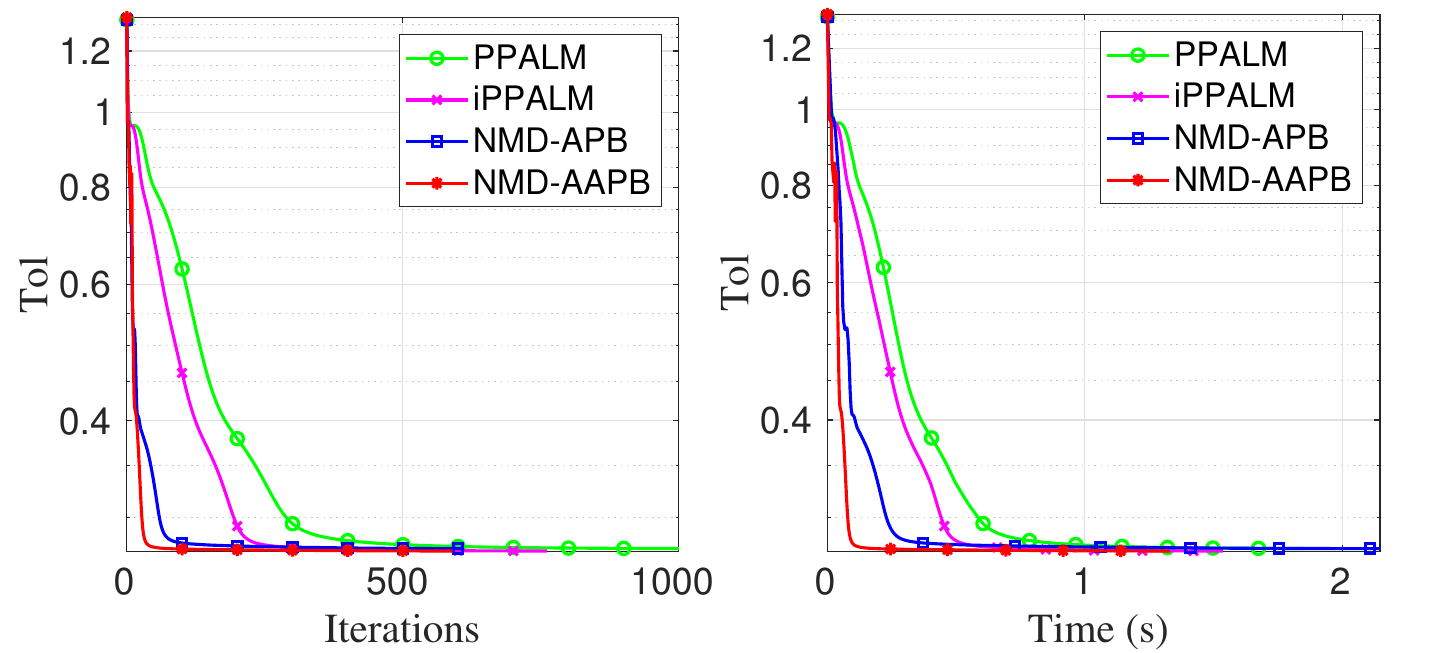}\\
		(a) Approximation rank $r=9$.\\
		\includegraphics[width=0.95\columnwidth]{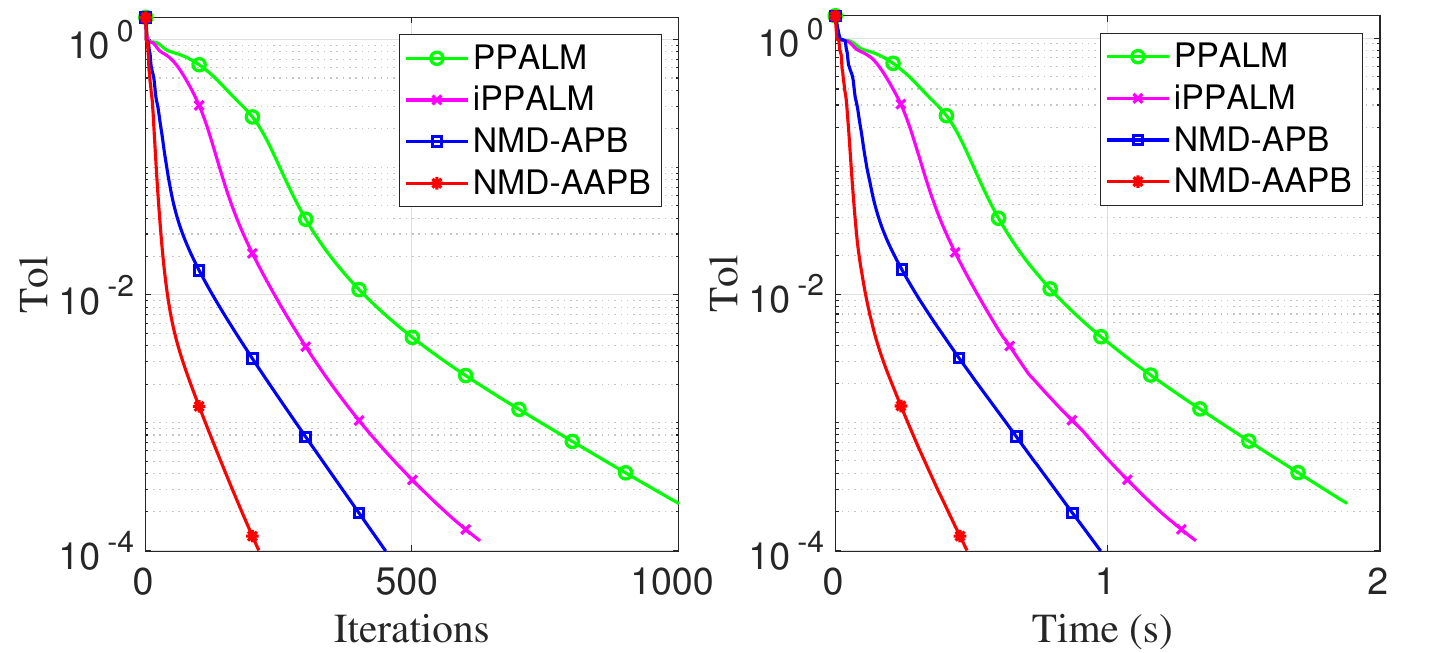}\\
		(b) Approximation rank $r=12$.\\
	\end{tabular}
	\caption{Numerical results for solving \eqref{L1-L1-model} with $m=1500,n=500, r^*=10$ by synthetic datasets.} 
	\label{syn_l1l1}
\end{figure}


Next, we evaluate our approach on three real-world datasets:  \emph{netz4504-dual},   \emph{mycielskian11}\footnote{Source: \url{https://sparse.tamu.edu/}}, and  the widely-used \emph{MNIST} dataset\footnote{Source: \url{http://yann.lecun.com/exdb/mnist/}}. The detailed characteristics and statistical properties of these datasets are presented in Table \ref{real_l1l1_details}.

\begin{table}[h!]
	\begin{center}
		\caption{The details for three real-world datasets applied in solving \eqref{L1-L1-model}.}   \label{real_l1l1_details}
		\begin{tabular}{c|c c c } 
			\hline 
			Data & $m$ & $n$ & sparsity ($\%$) \\\hline
			\emph{netz4504-dual} & 615 &615 & 0.62\\
			\emph{mycielskian11} & 1535 & 1535 & 5.72\\
			\emph{MNIST} & $1000\times10$ & 784 & 19.27\\\hline
		\end{tabular}
	\end{center}
\end{table}

We set $\max_{T}=30\,s$ and $\max_{K}=1000$ and show the numerical results   in Table \ref{real_l1l1_table}. It is evident from this table that, compared with   PPALM and iPPALM, the algorithm proposed in this paper achieves better numerical results.  
Furthermore,  NMD-AAPB converges faster than NMD-APB. 
The numerical performance is also confirmed in Figure~\ref{real_l1l1_r20}.

\begin{table}[h!]
	\begin{center}
		\setlength{\tabcolsep}{1mm}
		\caption{The relative error for three real-world datasets for solving optimization problem \eqref{L1-L1-model}.
		}
	\label{real_l1l1_table}
	\begin{tabular}{c|c|c c| c c} 
		\hline 
		Data & rank ($r$) & PPALM & iPPALM & NMD-APB & NMD-AAPB\\\hline
		\multirow{3}{*}{\emph{netz4504-dual}} & 10 & 9.0e-1 &6.2e-1 & 1.8e-1 & \textbf{1.5e-1}\\
		& 15 & 7.2e-1 &1.7e-1 & 4.9e-2& \textbf{6.2e-3} \\
		& 20 & 4.4e-1&2.4e-2 & 2.1e-2& \textbf{5.8e-4}\\ \hline
		\multirow{3}{*}{\emph{mycielskian11}} & 15 & 6.0e-1 &2.4e-1 & 1.3e-1 & \textbf{9.3e-2}\\
		& 25 & 3.0e-1 &1.3e-1 & 8.3e-2 & \textbf{6.7e-2} \\
		& 35 & 2.1e-1&9.1e-2 & 6.2e-1& \textbf{4.1e-2}\\ \hline
		\multirow{2}{*}{\emph{MNIST}} &25 & 4.1e-1 &3.5e-1 & 3.3e-1& \textbf{3.0e-1}\\
		& 35 & 3.8e-1 &3.0e-1 & 2.7e-1& \textbf{1.7e-1} \\
		$m=1000\times 10$& 45 & 3.3e-1&2.8e-1 & 2.4e-1& \textbf{1.1e-1} \\ \hline
	\end{tabular}
\end{center}
\end{table}

\begin{figure}[!ht]
\setlength\tabcolsep{2pt}
\centering
\begin{tabular}{c}
	\includegraphics[width=0.95\columnwidth]{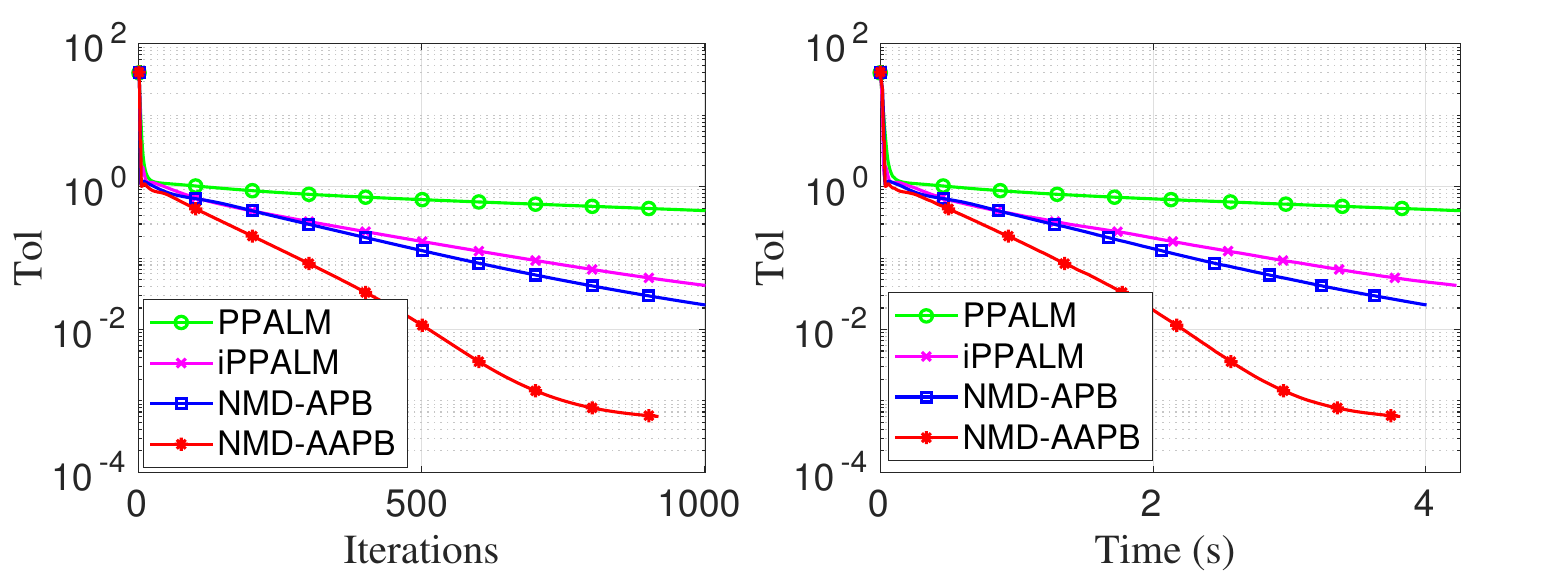}\\
	(a) \emph{netz4504-dual} with $r=20$. Left: Relative error. Right: Time (seconds).\\
	\includegraphics[width=0.95\columnwidth]{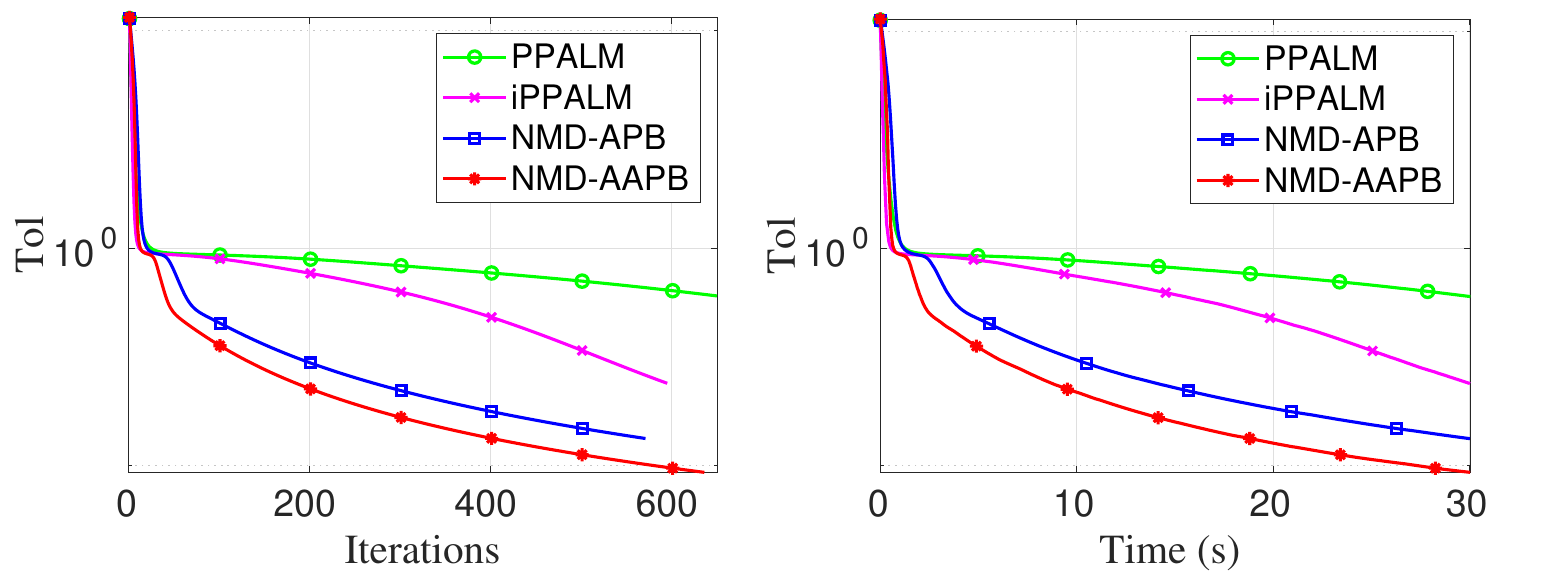}\\
	(b) \emph{mycielskian11} with $r=15$. 
	Left: Relative error. Right: Time (seconds).\\
\end{tabular}
\caption{Numerical results for solving \eqref{L1-L1-model} by two real-world datasets.  } 
\label{real_l1l1_r20}
\end{figure}

\bigskip

\section{Conclusion} \label{conclusion}
In this paper, we investigated feature extraction and pattern recognition in sparse nonnegative matrices. While the original data may not inherently exhibit low-rank properties, we explored  the intrinsic low-rank characteristics through the application of the ReLU activation function. 
We introduced a novel nonlinear matrix decomposition framework that incorporates comprehensive regularization terms, capable of addressing critical tasks such as graph-regularized clustering and sparse NMF basis compression. Specifically, we developed an accelerated alternating partial Bregman (AAPB) algorithm tailored for solving this model.  
From a theoretical perspective, we established both the sublinear and global convergence properties of the proposed algorithm. Extensive numerical experiments on graph-regularized clustering and sparse NMF basis compression tasks demonstrated the effectiveness and efficiency of our proposed model and algorithm.

The proposed model shows distinctive advantages in the following conditions: high sparsity in the dataset and the absence of closed-form solutions for the subproblems related to factor matrices. In cases where closed-form solutions are attainable, i.e., the regularization terms in \eqref{NMD-T-Obj} are zeros or Tikhonov regularizations, the alternative algorithms such as 3B-NMD and NMD-TM may exhibit superior performance. 


\appendix
\section{iPPALM framework}\label{ippalm_sec}
We present the framework of the inertial partial proximal alternating linearization minimization (iPPALM) algorithm for the optimization problem \eqref{NMD-T-Obj} (Algorithm \ref{iPPALM}). Note that when $\beta_{k}=0$ in Algorithm \ref{iPPALM}, it reduces to the partial proximal alternating linearization minimization (PPALM) method. Using the convergence analysis framework from \cite{PockS16, WangHZ24}, we can also establish the global convergence result for the iPPALM algorithm.

\begin{algorithm}[!ht]
\caption{\textbf{iPPALM:} Inertial partial proximal alternating linearization minimization for the optimization problem \eqref{NMD-T-Obj}}
\label{iPPALM}
{\bfseries Input:} $M$, $r$, $I_{+}$, $I_{0}$, and $K$. Bregman distance $\psi$. \\
{\bfseries Initialization:} $U^{0}$,  $V^{0}$, $X^{0}=U^{0}V^{0}$, and set $W_{i,j}=M_{i,j}$ for $(i,j)\in I_{+}$.
\begin{algorithmic}[1] 
	\For {$k=0,1,\dots K$} 
	\State Compute $W_{i,j}^{k+1}=\min(0,X_{i,j}^{k})$ for $(i,j)
	\in I_{0}$.
	\State Compute $\bar{U}^{k}= U^{k}+\beta_{k}(U^{k}-U^{k-1})$ with $\beta_{k}\in[0,1)$, let $\lambda_{1}^{k}>0$ and update
	\begin{align*}
		U^{k+1}\in\underset{U}{\text{argmin}}\, &\langle \nabla_{U}F(\bar{U}^{k},V^{k},W^{k+1}),U-\bar{U}^{k}\rangle\\
		&+\frac{1}{2\lambda_{1}^{k}}\|U-\bar{U}^{k}\|_{F}^{2}+H_{1}(U).
	\end{align*}
	\State Compute $\bar{V}^{k}= V^{k}+\beta_{k}(V^{k}-V^{k-1})$, let $\lambda_{2}^{k}>0$ and update
	\begin{align*}
		V^{k+1}\in\underset{V}{\text{argmin}}\, &\langle \nabla_{V}F(U^{k+1},\bar{V}^{k},W^{k+1}),V-\bar{V}^{k}\rangle\\
		&+\frac{1}{2\lambda_{2}^{k}}\|V-\bar{V}^{k}\|_{F}^{2}+H_{2}(V).
	\end{align*}
	\State Compute $X^{k+1}=U^{k+1}V^{k+1}$.
	\EndFor
\end{algorithmic} 
{\bfseries Output:}  $X^{k+1}$.
\end{algorithm}

\section*{Declarations}
{\bf Funding:} {This research is supported by the National Key R\&D Program of China (No. 2021YFA1003600), the National Natural Science Foundation of China (NSFC) grants 12131004, 12401415, 12471282, the R\&D project of Pazhou Lab (Huangpu) (Grant no. 2023K0603, 2023K0604), the Natural Science Foundation of Hunan Province (No. 2025JJ60009), and the Fundamental Research Funds for the Central Universities (Grant No. YWF-22-T-204)}.

\noindent{\bf Competing interests:} The authors have no competing interests to declare that are relevant to the content of this article.

\bibliographystyle{abbrv}
\bibliography{Ref_ReLU_NMD}
\end{document}